\documentclass[10pt]{article}
\usepackage[preprint]{tmlr}
\usepackage{amsmath,amsfonts,bm}

\def\eqref#1{equation~\ref{#1}}
\def\1{\bm{1}}

\DeclareMathAlphabet{\mathsfit}{\encodingdefault}{\sfdefault}{m}{sl}
\SetMathAlphabet{\mathsfit}{bold}{\encodingdefault}{\sfdefault}{bx}{n}

\usepackage{hyperref}
\usepackage{url}
\usepackage{graphicx}
\usepackage{booktabs}
\usepackage{caption}
\usepackage{subcaption}
\usepackage{amsthm}
\usepackage{amsmath}
\usepackage{amssymb}
\usepackage{mathtools}
\usepackage{xcolor}
\usepackage{multirow}
\usepackage{array}
\usepackage{microtype}
\usepackage{enumitem}
\usepackage{tikz}
\usetikzlibrary{arrows.meta, positioning, shapes.geometric, fit, backgrounds}
\usepackage{placeins}

\newtheorem{definition}{Definition}

\newtheorem{remark}{Remark}
\newtheorem{proposition}{Proposition}
\newtheorem{lemma}{Lemma}

\newtheorem{corollary}{Corollary}

\newcommand{\SHD}{\textsc{SHD}}
\newcommand{\DAGMA}{\textsc{Dagma}}
\newcommand{\NOTEARS}{\textsc{Notears}}
\newcommand{\GOLEM}{\textsc{golem}}
\newcommand{\ARelax}{\textsc{AdaptiveRelax}}
\newcommand{\Lfit}{\mathcal{L}_{\mathrm{fit}}}
\newcommand{\Laug}{\mathcal{L}_{\mathrm{aug}}}
\newcommand{\Lcov}{\mathcal{L}_{\mathrm{cov}}}
\newcommand{\Llstsq}{\mathcal{L}_{\mathrm{lstsq}}}
\newcommand{\Lloglik}{\mathcal{L}_{\mathrm{loglik}}}
\newcommand{\corr}{\mathrm{corr}}

\title{%
Guide, Not Bind: Why Defeasible Priors Fail in\\
Augmented Lagrangian Causal Discovery%
}

\author{\name Sairam Sundararaman\thanks{These authors contributed equally to this work.} \email pes1ug23am257@pesu.pes.edu \\
      \addr Department of Computer Science (AI \& ML)\\
      PES University
      \AND
      \name Sara Girdhar\footnotemark[1] \email pes1ug23am273@pesu.pes.edu \\
      \addr Department of Computer Science (AI \& ML)\\
      PES University
      \AND
      \name Manit Narasimha Murthy \email pes1ug23cs348@pesu.pes.edu\\
      \addr Department of Computer Science \\
      PES University
      \AND
      \name Samrudh N \email pes1ug23cs507@pesu.pes.edu\\
      \addr Department of Computer Science\\
      PES University
      \AND
      \name Bhaskarjyoti Das \email bhaskarjyotidas@pes.edu\\
      \addr Department of Computer Science (AI \& ML)\\
      PES University
      }

\begin{document}
\maketitle

% =====================================================================
\begin{abstract}
Differentiable causal discovery methods increasingly encode expert priors as forbidden-edge
constraints enforced by an Augmented Lagrangian (ALM) penalty, on the assumption that a
data-adaptive relaxation mechanism will discount and eventually override a rule the data
consistently contradicts. We show this design, which we call \emph{guide, not bind}, fails for
two independent, precisely characterized reasons, and that directly repairing both restores it
only partially. First, sequential penalty-ramping ALM suppresses a wrongly-forbidden true edge
before any counterfactual check can detect it: we give three necessary conditions any adaptive
relaxation must satisfy to avoid this (Proposition~\ref{prop:conditions}), prove that DADU---the
natural relaxation rule this paper introduces as the object of study---violates all three
(Corollary~\ref{cor:dadu_failure}), and confirm the failure
across 3{,}072 training runs spanning graphs from 4 to 32 nodes, where a single wrong prior
suppresses a true edge in 87--97\% of trials under DADU. Second, and independent of
any fix to the mechanism, we prove in closed form that the standard correlation-matching objective
ties a true edge and its reverse to an identical cost of exactly $2r^2$ (Lemma~\ref{lem:tie}), not
because the underlying equal-variance model is unidentifiable, but because normalizing to
correlation discards exactly the variance information that would make it identifiable; covariance
matching instead separates the two directions by a provable margin of at least $w_0^4$
(Lemma~\ref{lem:separation}). We build and test the fix our own diagnosis specifies---a relaxation
operator satisfying all three necessary conditions, combined with covariance matching---and find
that across 768 paired trials on identical graphs it recovers a wrongly-forbidden edge 27 to 52
times more often than DADU does ($p<10^{-5}$), yet the majority of the suppressed
edge's weight is still absorbed by its unconstrained reverse direction rather than restored to the
correct one. Finally, we ask whether either obstruction is an artifact of the correlation-matching
objective our testbed uses, by re-running the full sweep under the actual least-squares and
likelihood objectives \NOTEARS{}/\DAGMA{} and \GOLEM{} use: the identifiability tie is indeed
correlation-specific and disappears under both, exactly as our diagnosis predicts, but the
suppression mechanism is not---it worsens under a likelihood objective, which suppresses the
wrongly-forbidden edge in 96--100\% of trials, more often than correlation matching does. A
provably-identifying objective, it turns out, is neither necessary nor sufficient protection
against a penalty schedule that never gave the data a chance to be heard.
\end{abstract}

\section{Introduction}
\label{sec:intro}

Consider a practitioner building a differentiable causal discovery system who holds one piece of
domain knowledge with less than full confidence---say, that material does not determine shape. The
standard advice against encoding this as a hard constraint is well known: a hard constraint is
unforgiving, and a single incorrect one cannot be corrected by any amount of data. The
better-motivated design is to treat such a rule as a prior---something the learner trusts by
default, discounts as evidence accumulates against it, and overrides outright once that evidence
is consistent and strong. This is not a new idea: it is the logic behind empirical-Bayes shrinkage
and spike-and-slab priors, where a point-mass belief yields to sufficient likelihood evidence
rather than remaining fixed in advance. Stated this way, the design is less a choice than a piece of common sense: a prior should guide the learner, not bind it.

This instinct has a natural implementation in differentiable causal discovery. Such methods already
recover DAGs by gradient descent on continuous relaxations of acyclicity
\citep{zheng2018dags,bello2022dagma,ng2020role,lachapelle2020gradient,yu2021dags}, and already carry
the machinery to enforce a hard constraint on that relaxation: an Augmented Lagrangian (ALM), the
same tool these methods use to enforce acyclicity itself
\citep{hestenes1969multiplier,powell1969method,bertsekas2014constrained}. So encoding a forbidden
edge costs nothing new---penalize the learner for violating it, and grow that penalty over training
with an ALM. One further piece makes the constraint defeasible: periodically check whether the data
supports the forbidden edge regardless, and relax the penalty if it does. On paper, this is exactly
\emph{guide, not bind}. It has the right shape, built entirely from parts these methods already
use, and it should work.

It does not, and the first reason is not something a better threshold can fix. The adaptive check
works by asking a counterfactual: if this edge were removed right now, how much worse would the
fit become? That is a meaningful question to ask of an edge still near its data-supported value.
It is not a meaningful question to ask of an edge the penalty has already driven to near
zero---an edge that is already gone in all but name will correctly report that removing it costs
almost nothing, regardless of what the data actually wanted. Sequential penalty-ramping ALM grows
the constraint geometrically, and does so before the adaptive check ever gets a chance to look; by
the time the check runs, there is usually nothing left for it to see. We call this the \emph{early
suppression trap}. Sweeping the correction threshold across a wide range changes nothing, because
the evidence the threshold is meant to act on has already been erased before any threshold is
applied. This part of the failure has nothing to do with causal discovery specifically: it is a
statement about what happens when a penalty is allowed to race ahead of the check meant to keep it
honest, and it recurs anywhere a defeasible constraint is enforced by a penalty that grows on a
fixed schedule---the classical ALM convergence theory that guarantees the multiplier diverges when
a constraint is genuinely incompatible with the data \citep{bertsekas2014constrained} says nothing
about what an adaptive relaxation operator sees while that divergence is still in progress, and it
is exactly that blind spot the early suppression trap exploits.

Suppose this is fixed: suppose the check runs earlier, evaluated before the penalty has done any
damage, so it always sees the edge at its true, data-supported value. This should settle the
matter. It does not, and the second reason is more interesting than the first, because it survives
a perfect fix to the first one. An early, honestly-timed check compares how well the model fits
with a candidate edge in each of its two possible directions. If that comparison cannot tell a true
edge from its reverse in the first place, evaluating it early buys nothing---the check will
faithfully report that both directions look equally good, because, under the objective these
methods use, they are. This should be surprising: the model class in question assumes every
variable has the same noise variance, and it is a known fact that under that assumption, causal
direction is in principle recoverable from data alone \citep{petersbuhlmann2014}. The tie, in
other words, is not inherited from an unidentifiable model. Something else discards the very
signal that would make identification possible, and it turns out to be one specific, avoidable
modeling choice: comparing normalized correlation instead of raw covariance throws away exactly
the variance asymmetry that equal-variance identifiability depends on. This is a narrower claim
than it might sound: correlation matching is one fitting objective among several used in this
literature, not the only one, and part of what this paper does is find out whether the tie is a
property of that specific choice or of forbidden-edge enforcement more broadly---a question the
diagnosis alone cannot answer and that we return to with a direct test in
Section~\ref{sec:emp-obstruction2b}.

Two questions remain even once both diagnoses are in hand. Is a mechanism satisfying our three
necessary conditions actually buildable, or only specifiable? And if it is built, and combined with
covariance matching as our own analysis prescribes, does it recover a wrongly-forbidden true edge,
or does some third obstacle appear once the first two are cleared? We build such a mechanism and
test it directly. The answer is instructive rather than clean: across matched pairs of identical
random graphs, the combined fix recovers the wrongly-forbidden edge far more often than DADU
does---a difference too large to be chance. But recovery remains the minority
outcome. In most trials, the weight the forbidden penalty pushes out of the true edge is absorbed
by its unconstrained reverse, not restored to the correct direction---the tie our own Lemma~1
predicts, appearing not as a closed-form curiosity but as the dominant failure mode of a system
built specifically to avoid it. Fixing the two obstructions we diagnose is necessary. It is not
sufficient.

A further question follows directly from how the paper is set up so far: everything above is
proved and measured on a single fitting objective, correlation matching, that we chose for the
testbed rather than one drawn from an existing, deployed system. Nothing in the argument requires
this choice, and the actual objectives used by \NOTEARS{} \citep{zheng2018dags}, \DAGMA{}
\citep{bello2022dagma}, and \GOLEM{} \citep{ng2020role} fit least-squares reconstruction or
likelihood on raw, unnormalized data instead. A reader who accepts our diagnosis in
Remark~\ref{rem:pb}---that the tie comes specifically from discarding scale information, not from
forbidden-edge enforcement in general---should expect that tie to vanish once scale is put back.
We test this directly rather than leaving it as a plausible inference: Section~\ref{sec:empirical}
re-runs the entire sweep under a \NOTEARS{}/\DAGMA{}-style least-squares objective and a
\GOLEM{}-EV-style likelihood objective, both written as exact reconstructions of the objectives
those methods actually optimize. The identifiability side of the diagnosis survives this test
cleanly. The suppression side does not, and not in the reassuring direction: the likelihood
objective, which sidesteps the tie exactly as predicted, turns out to suppress the wrongly-forbidden
edge \emph{more} often than correlation matching does, not less. Untangling why a fix to one
obstruction can leave the other one worse is a large part of what the rest of this paper is about.

This paper gives all four findings---the two failures, the partial fix, and the test of whether
either failure is specific to our testbed's objective---a precise, tested treatment.

\begin{enumerate}[leftmargin=1.4em]
  \item \textbf{The early suppression trap} (Section~\ref{sec:mechanism}): three necessary
    conditions for any data-adaptive relaxation to avoid irreversible suppression
    (Proposition~\ref{prop:conditions}), and a proof that DADU---the natural relaxation rule this
    paper introduces as the object of study, not a rule drawn from an existing published
    implementation---violates all three simultaneously once suppression begins
    (Corollary~\ref{cor:dadu_failure}).
  \item \textbf{An exact identifiability barrier, correctly attributed}
    (Section~\ref{sec:identifiability}): a closed-form proof that correlation matching ties a true
    edge and its reverse exactly (Lemma~\ref{lem:tie}); a precise account of why this follows from
    discarding marginal-variance information rather than from any inherent non-identifiability of
    the underlying model \citep{petersbuhlmann2014}; a closed-form proof that covariance matching
    instead provably separates the two directions (Lemma~\ref{lem:separation}); and a
    generalization of the same mechanism to any Markov-equivalence-preserving edge reversal
    (Proposition~\ref{prop:general}).
  \item \textbf{A combined mechanism, built and tested} (Section~\ref{sec:combined}): a relaxation
    operator that operationalizes all three necessary conditions of Proposition~\ref{prop:conditions}
    (Proposition~\ref{prop:arelax}), evaluated together with covariance matching on 768 paired
    synthetic instances. The combined fix significantly outperforms DADU
    ($p<10^{-5}$) but recovers the correct edge in the minority of trials, with the unconstrained
    reverse edge remaining the dominant attractor.
  \item \textbf{A test of whether either obstruction is testbed-specific}
    (Section~\ref{sec:emp-obstruction2b}): the full sweep re-run under a least-squares objective
    matching \NOTEARS{}/\DAGMA{} and a likelihood objective matching \GOLEM{}-EV, confirming that
    Obstruction~II is specific to correlation matching while Obstruction~I is not---and that the
    likelihood objective, the one this literature already treats as the safe choice, suppresses
    more often than the objective we spend the rest of the paper criticizing.
  \item \textbf{A broad empirical validation} (Section~\ref{sec:empirical}): all of the above
    evaluated across 6{,}144 training runs spanning graphs from 4 to 32 nodes, two edge densities,
    three noise scales, both equal-variance and heteroscedastic noise, and four fitting objectives,
    reporting only what this controlled synthetic sweep directly supports.
\end{enumerate}

Two things are worth stating plainly before the technical sections begin, since they bound what the
rest of the paper can and cannot claim. The first is scope. The mechanism failure and the
identifiability barrier apply to any ALM-based structure learner with sequential penalty ramping
and a fitting objective from the family we test; all empirical results in this paper use synthetic
linear Gaussian SEMs with known ground-truth structure, chosen specifically so that
Sections~\ref{sec:mechanism} and~\ref{sec:identifiability}'s claims can be checked exactly against
a known answer rather than estimated from an uncertain one. This buys precision at the cost of not
showing what either obstruction looks like once real-world measurement noise, concept extraction,
or a non-Gaussian data-generating process is also in the loop---a cost we return to and itemize
fully in Section~\ref{sec:limitations}. We are explicit throughout about which claims are proved,
which are illustrated, and which remain open, and that section collects the open ones in one place.

The second is where this paper sits relative to what is already known. Differentiable causal
discovery's Augmented Lagrangian machinery is borrowed wholesale from classical constrained
optimization \citep{hestenes1969multiplier,powell1969method,bertsekas2014constrained}, whose theory
already tells us what happens if a constraint is simply wrong---the multiplier
diverges---but says nothing about a relaxation operator racing that divergence in real time before
it has a chance to detect the mistake, which is precisely the gap Obstruction~I occupies. Encoding
prior knowledge as a forbidden edge has a long history in discrete, constraint-based causal
discovery \citep{spirtes2000causation,borboudakis2017incorporating,meek1995causal,
andrews2020completeness}, including direct evidence that one incorrect constraint can cascade
through an entire recovered graph \citep{constantinou2021impact}; what is new here is the specific
failure that arises once the constraint is not fixed in advance but is instead enforced by a
penalty that grows while training runs. On the identifiability side, the fact that linear Gaussian
SEMs are recoverable only up to Markov equivalence in general
\citep{peters2017elements,vermapearl1990,chickering2002} is well known, as are three routes around
it---non-Gaussian noise \citep{shimizu2006lingam}, nonlinearity \citep{hoyer2009nonlinear}, and
interventional or temporal data \citep{hauser2012characterization,lippe2022citris}---but the fourth
route, equal-variance identifiability \citep{petersbuhlmann2014}, is less widely appreciated, and to
our knowledge no prior work has asked what a correlation-based fitting objective costs a model class
that already satisfies it. We discuss the closest related strands in optimization stability
\citep{nazaret2023stable,waxman2024dagma,yi2025robustness}, neuro-symbolic learning
\citep{manhaeve2018deepproblog,de2019neuro,marconato2023not}, and soft, score-integrated priors
\citep{darvariu2024llm} at length in Section~\ref{sec:related}, after the technical machinery those
comparisons depend on has been introduced.

% =====================================================================
\section{Background}
\label{sec:background}

This section fixes the objects the rest of the paper manipulates: structural equation models and
the differentiable relaxation used to fit them (\S\ref{sec:bg-sem}), the Augmented Lagrangian
machinery used to enforce a prior constraint on them (\S\ref{sec:bg-alm}), and the precise sense in
which a fitted model's causal direction can, or cannot, be recovered from data alone
(\S\ref{sec:bg-ident}). Each subsection states one definition and the one fact about it this paper
depends on; a reader already fluent in differentiable causal discovery can proceed directly to
Section~\ref{sec:method}, which specializes everything here to this paper's exact objective.

\subsection{Structural Equation Models and Differentiable Discovery}
\label{sec:bg-sem}

Recovering a causal graph from data requires first fixing what a causal graph is and what
recovering one means.

\begin{definition}[Linear Gaussian SEM]
\label{def:sem}
A \emph{linear Gaussian structural equation model} (SEM) over $\mathbf z=(z_1,\dots,z_d)\in\mathbb
R^d$ is a pair $(W,\Sigma_\varepsilon)$ with $W\in\mathbb R^{d\times d}$ zero-diagonal,
$\Sigma_\varepsilon\succ0$, $I-W$ invertible, satisfying
\begin{equation}
    \mathbf z = W^\top\mathbf z + \boldsymbol\varepsilon, \qquad
    \boldsymbol\varepsilon\sim\mathcal N(\mathbf 0,\Sigma_\varepsilon).
    \label{eq:bg_sem}
\end{equation}
$W_{ij}\neq0$ is read as: variable $j$ is a direct cause of variable $i$. The associated graph
places a directed edge $j\to i$ for every nonzero $W_{ij}$; the model is \emph{acyclic} iff this
graph is a DAG, equivalently iff $W$ is permutation-similar to a strictly upper-triangular matrix.
\end{definition}

Solving Eq.~\ref{eq:bg_sem} for $\mathbf z$ gives the implied covariance
$\Sigma_\mathrm{SEM}(W)=(I-W)^{-1}\Sigma_\varepsilon(I-W)^{-\top}$---the only property of
$(W,\Sigma_\varepsilon)$ an i.i.d.\ sample of $\mathbf z$ can constrain, and consequently the only
handle any estimation procedure has on $W$. \emph{Causal discovery} recovers $W$, in particular its
direction, from such a sample without ever intervening on any $z_i$. Classical algorithms search
directly over the discrete space of DAGs, whose size grows super-exponentially in $d$;
\emph{differentiable} causal discovery replaces this search with continuous optimization by
relaxing acyclicity to a smooth regularizer $h(W)\geq0$, zero iff $W$ is acyclic, and minimizing a
smooth fit loss $\ell(W)$ plus $\beta h(W)$ by gradient descent \citep{zheng2018dags,
bello2022dagma}. We use \DAGMA{}'s log-determinant regularizer throughout,
$h^s(W)=-\log\det(sI-W\circ W)+d\log s$ \citep{bello2022dagma}; Section~\ref{sec:method} fixes
$\ell$ to the specific choice this paper's core results depend on, and
Section~\ref{sec:emp-obstruction2b} substitutes two alternatives drawn directly from the
literature.

\subsection{Augmented Lagrangian Enforcement of Prior Constraints}
\label{sec:bg-alm}

Practitioners frequently hold knowledge that some entries of $W$ are zero---a physical fact, such
as material not determining shape---before observing any data. Encoding it requires a mechanism
for enforcing a hard constraint on a continuous optimization problem.

\begin{definition}[Augmented Lagrangian enforcement]
\label{def:alm}
Let $M\in\{0,1\}^{d\times d}$ mark entries of $W$ known to be zero, and let
$c(W):=\|W\circ M\|_1$. The \emph{Augmented Lagrangian method} (ALM)
\citep{hestenes1969multiplier,powell1969method,bertsekas2014constrained} enforces $c(W)=0$ by
minimizing, jointly with the base objective $\ell(W)+\beta h(W)$, the augmented term $\lambda\,c(W)
+ \tfrac{\rho}{2}c(W)^2$, alternating a gradient step on $W$ with dual ascent on the multiplier
$\lambda$,
\begin{equation}
    \lambda\leftarrow\lambda+\rho\,c(W), \qquad \rho\leftarrow\kappa\rho,\quad\kappa>1,
    \label{eq:bg_alm}
\end{equation}
with $\rho$ typically ramped geometrically (\emph{penalty ramping}) rather than held fixed.
\end{definition}

The one fact about Definition~\ref{def:alm} this paper depends on is due to
\citet{bertsekas2014constrained}: if $c(W)=0$ is incompatible with the unconstrained minimizer of
$\ell+\beta h$, the multiplier sequence diverges, $\lambda_t\to\infty$. Every standard use of ALM---
in constrained optimization generally \citep{nandwani2019primal,chamon2020probably,
cotter2019optimization,achiam2017constrained} and in differentiable causal discovery specifically
\citep{zheng2018dags,bello2022dagma}---presumes $c(W)=0$ is correctly specified, in which case this
divergence is precisely the intended behavior: an always-true constraint should be enforced ever
more strictly. A \emph{defeasible} constraint is one where $c(W)=0$ might itself be wrong, and where
the enforcement mechanism must therefore also be able to relax $\lambda$ when the data disagrees;
Section~\ref{sec:method} specifies the exact relaxation rule this paper studies, and
Section~\ref{sec:mechanism} shows what Bertsekas's divergence result implies for that rule once
$\rho$ is ramped geometrically rather than held fixed.

\subsection{Markov Equivalence and Identifiability}
\label{sec:bg-ident}

Two SEMs with different $W$ can imply the same distribution over $\mathbf z$. When this happens, no
amount of observational data---however collected or analyzed---can determine which one generated
it; the ambiguity is a property of the data, not a limitation of any particular estimator.

\begin{definition}[Markov equivalence]
\label{def:markov-equiv}
Two DAGs are \emph{Markov equivalent} if every distribution consistent with one is consistent with
the other under some choice of parameters. For DAGs, this holds exactly when the graphs share a
skeleton (the same edges, ignoring direction) and the same v-structures \citep{vermapearl1990}, a
characterization made algorithmically complete by \citet{meek1995causal,chickering2002,
andrews2020completeness}.
\end{definition}

Under a general noise covariance $\Sigma_\varepsilon$, a linear Gaussian SEM is identifiable only
up to its Markov equivalence class \citep{peters2017elements}: reversing an edge, in general,
leaves an equally good explanation running the other way. Three routes are known to restore
identifiability by breaking this symmetry: non-Gaussian noise \citep{shimizu2006lingam},
nonlinearity \citep{hoyer2009nonlinear}, and interventional or temporal data
\citep{hauser2012characterization,lippe2022citris}. A fourth route, central to this paper, requires
none of these.

\begin{definition}[Equal-variance identifiability]
\label{def:ev-ident}
If $\Sigma_\varepsilon=\sigma^2I$---every exogenous noise term has the same variance---
\citet{petersbuhlmann2014} show the DAG is generically identifiable from $\Sigma_\mathrm{SEM}(W)$
alone: a variable further downstream in the causal order accumulates strictly greater marginal
variance under equal innovation variance, and comparing marginal variances across nodes recovers
direction, with no non-Gaussianity or intervention required.
\end{definition}

Definition~\ref{def:ev-ident} matters here because $\Sigma_\varepsilon=I$ is already the noise
model Section~\ref{sec:method} assumes: the question this paper asks is therefore not whether that
model is identifiable---it is---but whether a given fit loss $\ell$ actually exploits the signal
Definition~\ref{def:ev-ident} says is present. Section~\ref{sec:identifiability} shows that fitting
normalized correlation rather than covariance discards it. Correlation-matching is not the only
$\ell$ used in differentiable causal discovery---\NOTEARS{}, \DAGMA{}, and \GOLEM{}
\citep{ng2020role} fit least-squares or likelihood objectives that retain scale information---but
it is the choice this paper's core testbed makes, and Section~\ref{sec:identifiability} traces the
failure to that specific, avoidable step rather than to differentiable causal discovery generally.
Section~\ref{sec:emp-obstruction2b} makes this distinction load-bearing rather than rhetorical, by
substituting the two alternatives directly.

% =====================================================================
\section{Setup: Defeasible Priors via Augmented Lagrangian}
\label{sec:method}

We assume $n$ i.i.d.\ samples $\mathbf z_1,\dots,\mathbf z_n\in\mathbb R^d$ from a linear Gaussian
SEM (Eq.~\ref{eq:linear_sem} below) with unknown zero-diagonal weight matrix $W$, and form the
empirical covariance
\begin{equation}
    \hat\Sigma = \frac1n\sum_{i=1}^n(\mathbf z_i-\bar{\mathbf z})(\mathbf z_i-\bar{\mathbf z})^\top,
    \qquad \bar{\mathbf z}=\frac1n\sum_{i=1}^n\mathbf z_i.
    \label{eq:emp_cov}
\end{equation}
We assume a linear Gaussian structural equation model
\begin{equation}
    \mathbf z = W^\top \mathbf z + \boldsymbol\varepsilon, \qquad
    \boldsymbol\varepsilon \sim \mathcal N(\mathbf 0, I),
    \label{eq:linear_sem}
\end{equation}
with $W\in\mathbb R^{d\times d}$, zero diagonal, and acyclicity enforced softly via the \DAGMA{}
log-determinant regularizer $h^s(W)=-\log\det(sI-W\circ W)+d\log s$ \citep{bello2022dagma}. Our core
testbed fits $W$ by minimizing the Frobenius distance between \emph{normalized correlation}
matrices,
\begin{equation}
    \Lfit(W) = \left\| \corr\!\big(\Sigma_\mathrm{SEM}(W)\big) - \corr(\hat\Sigma) \right\|_F^2,
    \qquad \Sigma_\mathrm{SEM}(W) = (I-W)^{-1}(I-W)^{-\top}.
    \label{eq:lfit}
\end{equation}
Equation~\ref{eq:lfit} is the single most consequential design choice in Sections~\ref{sec:method}
through~\ref{sec:combined}: it is the source of Obstruction II, we return to it directly in
Section~\ref{sec:identifiability}, and Section~\ref{sec:emp-obstruction2b} tests what changes when
it is replaced.

\paragraph{The forbidden-edge prior.} A binary mask $M\in\{0,1\}^{d\times d}$ marks
physically-implausible edges. We enforce $\|W\circ M\|_1=0$ via an Augmented Lagrangian term with
per-edge dual variables $\lambda_{ij}$,
\begin{equation}
    \Laug = w_f \Lfit + \gamma\|W\|_1 + \lambda^\top(|W|\circ M)\mathbf 1
    + \tfrac{\rho}{2}\big(\|W\circ M\|_1\big)^2 + \beta_h h^s(W),
    \label{eq:laug}
\end{equation}
updated by standard dual ascent with a geometrically growing penalty,
\begin{equation}
    \lambda_{ij} \leftarrow \lambda_{ij} + \rho\cdot|W_{ij}|\cdot M_{ij}, \qquad
    \rho \leftarrow \kappa\rho, \quad \kappa>1.
    \label{eq:dual_ascent}
\end{equation}
\emph{Guide, not bind} requires a mechanism that can also move $\lambda_{ij}$ back down when a
forbidden edge turns out to be data-supported. The natural candidate---compute a counterfactual
cost $\Delta_{ij} = \Lfit(W\mid W_{ij}=0) - \Lfit(W)$ and relax $\lambda_{ij}$ whenever
$\Delta_{ij}$ exceeds a threshold $\delta$---is what we call the Data-Adaptive Dual Update (DADU):
\begin{equation}
    \lambda_{ij} \leftarrow \max(0,\ \lambda_{ij} - \eta_r\Delta_{ij})
    \quad \text{if } \Delta_{ij}\geq\delta, \qquad \text{else tighten as in
    Eq.~\ref{eq:dual_ascent}}.
    \label{eq:dadu}
\end{equation}
We call this rule DADU because it is the natural instantiation of guide-not-bind's own stated
logic---a data-adaptive check on the dual update---rather than a rule drawn from an existing
published implementation; we introduce it here as the object of study, precisely so that
Sections~\ref{sec:mechanism} and~\ref{sec:identifiability} can show \emph{why} the natural design
fails rather than only that some particular published system does. Sections~\ref{sec:mechanism}
and~\ref{sec:identifiability} show that Eq.~\ref{eq:dadu} fails to realize \emph{guide, not bind},
for two unrelated reasons: it evaluates $\Delta_{ij}$ too late to matter (Obstruction I), and even
evaluated in time, $\Delta_{ij}$ under Eq.~\ref{eq:lfit} cannot tell a data-supported edge from its
reverse (Obstruction II).

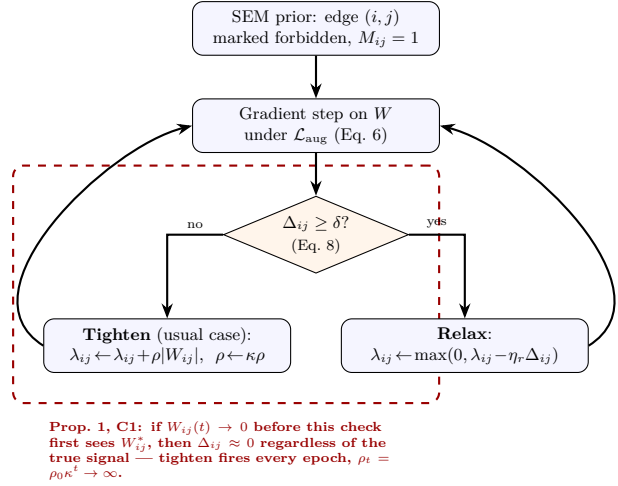
\begin{figure}[htbp]
\centering
\begin{tikzpicture}[
    scale=0.71, transform shape,
    box/.style={draw, rounded corners, align=center, minimum width=4.6cm,
                minimum height=1cm, font=\small, fill=blue!4},
    decision/.style={draw, diamond, aspect=2.4, align=center, font=\small,
                      fill=orange!8, inner sep=2pt},
    arr/.style={-{Stealth[length=2mm]}, thick},
    node distance=8mm and 12mm
]

\node[box] (prior) {SEM prior: edge $(i,j)$\\ marked forbidden, $M_{ij}=1$};
\node[box, below=of prior] (grad) {Gradient step on $W$\\ under $\Laug$ (Eq.~\ref{eq:laug})};
\node[decision, below=of grad] (check) {$\Delta_{ij}\geq\delta$?\\[1pt]\footnotesize (Eq.~\ref{eq:dadu})};
\node[box, below left=12mm and -4mm of check] (tighten)
      {\textbf{Tighten} (usual case):\\ $\lambda_{ij}\!\leftarrow\!\lambda_{ij}\!+\!\rho|W_{ij}|$,\ \ $\rho\!\leftarrow\!\kappa\rho$};
\node[box, below right=12mm and -4mm of check] (relax)
      {\textbf{Relax}:\\ $\lambda_{ij}\!\leftarrow\!\max(0,\lambda_{ij}\!-\!\eta_r\Delta_{ij})$};

\draw[arr] (prior) -- (grad);
\draw[arr] (grad) -- (check);
\draw[arr] (check) -| node[pos=0.25, above, font=\scriptsize] {no} (tighten);
\draw[arr] (check) -| node[pos=0.25, above, font=\scriptsize] {yes} (relax);
\draw[arr] (tighten.west) .. controls +(-16mm,10mm) and +(-16mm,-4mm) .. (grad.west);
\draw[arr] (relax.east) .. controls +(16mm,10mm) and +(16mm,-4mm) .. (grad.east);

\begin{scope}[on background layer]
\node[draw, dashed, thick, red!60!black, rounded corners, fit=(check)(tighten),
      inner sep=4mm] (trapbox) {};
\end{scope}
\node[below=2mm of trapbox, red!60!black, font=\scriptsize\bfseries, align=left,
      text width=6.6cm] (traplabel)
      {Prop.~\ref{prop:conditions}, C1: if $W_{ij}(t)\to0$ before this check
      first sees $W_{ij}^{*}$, then $\Delta_{ij}\approx0$ regardless of the
      true signal --- \textbf{tighten} fires every epoch,
      $\rho_t=\rho_0\kappa^t\to\infty$.};

\end{tikzpicture}
\caption{\textbf{The defeasible-prior enforcement loop this paper studies.} Once an edge $(i,j)$ is
marked forbidden ($M_{ij}=1$), every training epoch alternates a gradient step on the
structure-learning objective $\Laug$ (Eq.~\ref{eq:laug}) with a counterfactual check: does removing
edge $(i,j)$ right now cost the fit at least $\delta$ in loss (Eq.~\ref{eq:dadu})? If so, the
constraint is \emph{relaxed}. If not, it is \emph{tightened}, and the penalty weight $\rho$ grows
geometrically regardless of which branch is taken. This loop implements \emph{guide, not bind}
whenever the check can be trusted. The dashed region marks where that trust breaks down:
Proposition~\ref{prop:conditions}'s Condition C1 shows that once the forbidden edge has already
been driven near zero, the counterfactual check becomes uninformative by construction, and the
loop takes the \textbf{tighten} branch every remaining epoch even when $(i,j)$ was a true,
data-supported edge---the early suppression trap of Section~\ref{sec:mechanism}.}
\label{fig:setup}
\end{figure}

% =====================================================================
\section{Obstruction I: The Suppression Mechanism}
\label{sec:mechanism}

This section states three necessary conditions any data-adaptive relaxation must satisfy to avoid
irreversibly suppressing a wrongly-forbidden true edge, shows DADU violates all
three, and confirms the resulting failure on a fully controlled synthetic instance.

\begin{proposition}[Necessary conditions for data-adaptive relaxation]
\label{prop:conditions}
Let $W_{ij}^*$ denote the unconstrained optimum of $\Lfit$ and let $W_{ij}(t)$ denote the weight at
epoch $t$ under $\rho_t=\rho_0\kappa^t$. For any data-adaptive relaxation operator $\mathcal R$
applied to $\lambda_{ij}$ to prevent irreversible suppression of a data-supported forbidden edge,
the following are necessary under the first-order local analysis this proposition is stated in
(we return to the scope of that qualification in Section~\ref{sec:limitations}).

\smallskip\noindent\textbf{C1 (Pre-suppression evaluation).} $\Delta_{ij}$ must be evaluated at or
near $W_{ij}^*$ before penalty mass accumulates: if instead evaluated at a suppressed
$W_{ij}(t)\approx 0$, then $\Delta_{ij}(t)\approx 0$ for \emph{any} data-generating process, and
$\mathcal R$ cannot distinguish a genuinely absent edge from a suppressed data-supported one.

\smallskip\noindent\textbf{C2 (Rate condition).} The growth rate $\kappa$ must satisfy
\begin{equation}
    \kappa < 1 + \frac{\eta_r\cdot \partial\Lfit/\partial W_{ij}
    \big|_{W_{ij}^*}}{\rho_0\cdot W_{ij}^*}.
    \label{eq:rate_condition}
\end{equation}
If violated, the penalty accumulates faster than any additive relaxation step can counteract,
regardless of the signal strength $\Delta_{ij}$. Eq.~\ref{eq:rate_condition} is derived from a
first-order expansion of $\Delta_{ij}$ around $W_{ij}^*$, so it is necessary \emph{under that local
approximation}; a fully optimizer-agnostic version is not established here.

\smallskip\noindent\textbf{C3 (Operator condition).} $\mathcal R$ must be able to move
$\lambda_{ij}$ against the direction of tightening, via a slack-variable formulation decoupling the
dual variable from the penalty, or an explicit ceiling on $\lambda_{ij}$; the plain ascent step in
Eq.~\ref{eq:dual_ascent} provides neither.
\end{proposition}

\begin{proof}
See Appendix~\ref{app:prop1_proof}.
\end{proof}

\begin{corollary}[DADU violates all three conditions simultaneously]
\label{cor:dadu_failure}
The Data-Adaptive Dual Update, DADU (Eq.~\ref{eq:dadu}), violates C1, C2, and C3 of
Proposition~\ref{prop:conditions} simultaneously, for every threshold $\delta>0$, once suppression
begins.
\end{corollary}

\begin{proof}
See Appendix~\ref{app:cor1_proof}.
\end{proof}

The trap follows from the asymmetry between exponential tightening and linear relaxation. DADU's
relaxation step subtracts a quantity proportional to $W_{ij}(t)$; once $W_{ij}(t)\approx0$ this
vanishes. The suppression window closes at approximately
\begin{equation}
    T^* \approx \frac{\log(\alpha w_0/\rho_0)}{\log\kappa},
    \label{eq:window}
\end{equation}
where $w_0=W^*_{ij}$ is the unconstrained edge weight and $\alpha$ is the local curvature of
$\Lfit$ at $W^*_{ij}$; this is a first-order estimate, not a tight bound. After epoch $T^*$,
$\Delta_{ij}$ remains near zero for all subsequent epochs regardless of $\delta$, $\eta_r$, or
causal signal strength: recovery within any finite training horizon is impossible.
Section~\ref{sec:empirical} confirms this across graphs from 4 to 32 nodes: under a
wrongly-forbidden true edge, DADU suppresses it in 87--97\% of trials depending on graph size.

\subsection{Independent Synthetic Verification}
\label{sec:mechanism-synthetic}

The dynamics above are proved for the general case; here we watch them happen on a fully
controlled instance, with no additional variables and no observation noise beyond the model's own.

The true edge is $1\to2$ with weight $w_0=0.55$. Under the correct prior (the genuinely-absent edge
$2\to1$ forbidden), training recovers $W_{12}=0.550$ against the true $0.550$. Under the
wrong-forbidden condition (the true edge $1\to2$ forbidden instead), Figure~\ref{fig:suppression}
shows what Corollary~\ref{cor:dadu_failure} predicts happening in real time: $W_{12}(t)$ is driven
to $\Delta_{12}\approx0$ by epoch~14, and a $\delta$-sweep over $\{0.01,0.05,0.10,0.20,0.50\}$
changes the final outcome by less than $10^{-3}$---the failure genuinely does not depend on where
the threshold is set.

\begin{figure}[htbp]
\centering
\includegraphics[width=\textwidth]{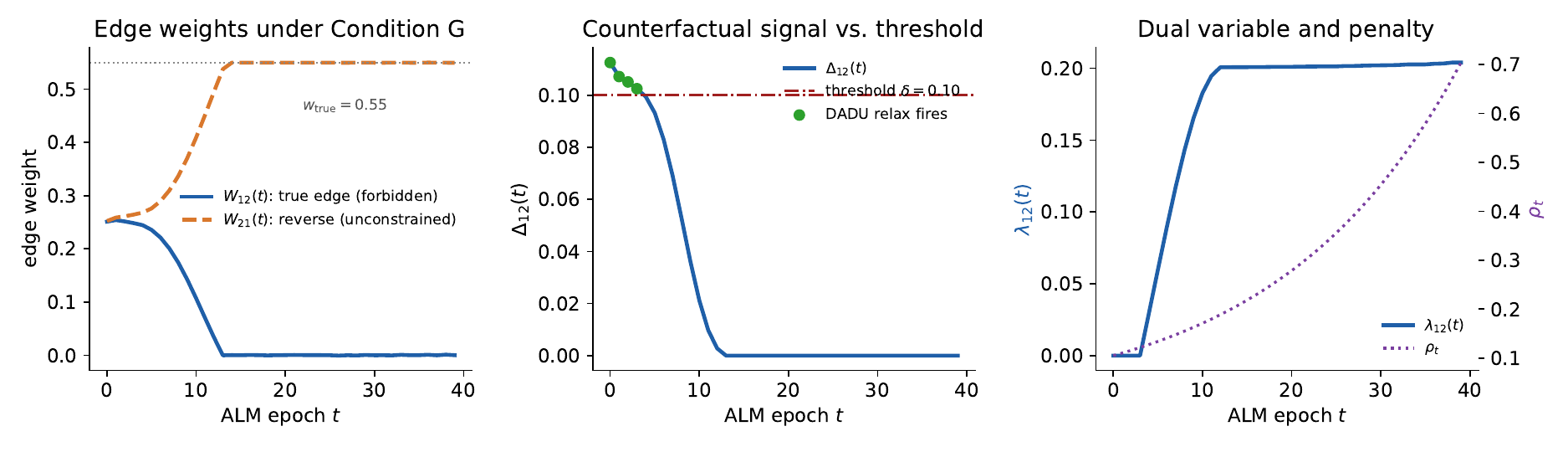}
\caption{\textbf{The early suppression trap on a fully controlled synthetic instance}
(Section~\ref{sec:mechanism-synthetic}), true edge $1\to2$ with $w_0=0.55$, the wrong-forbidden
condition, $\delta=0.10$. \emph{Left:} the forbidden true edge $W_{12}(t)$ collapses to zero by
epoch 14 while the unconstrained reverse edge $W_{21}(t)$ grows to the same magnitude as
$w_0$---Lemma~\ref{lem:tie}'s tie appearing inside a live optimization trajectory. \emph{Center:}
the counterfactual signal $\Delta_{12}(t)$ against the threshold $\delta$; green markers show where
DADU's relax branch fires, all early, before the reverse edge has taken over the fit.
\emph{Right:} the dual variable $\lambda_{12}(t)$ and penalty $\rho_t$; $\lambda_{12}$ plateaus once
$W_{12}\approx0$, since tightening a weight already near zero adds almost nothing further.}
\label{fig:suppression}
\end{figure}
\FloatBarrier

The reverse edge $W_{21}(t)$ converges to $0.5496$---the same magnitude as $w_0$, not an arbitrary
value. This is Lemma~\ref{lem:tie} manifesting inside an actual optimization trajectory: because
the correlation implied by a weight $w$ is direction-symmetric, once the true edge is suppressed,
$\Lfit$ is satisfied just as well by routing the same weight through the wrong direction, and the
unconstrained reverse edge does exactly that. At $\delta=0.01$, the most permissive threshold
tested, DADU's relax branch fires in 24 of 40 epochs, and the true edge is suppressed anyway: by
the time relaxation triggers, the reverse edge has already absorbed the job of fitting the
correlation, leaving no gradient pressure to pull $W_{12}$ back up. Relaxation firing is not the
same as relaxation working. Section~\ref{sec:empirical} shows the same pattern holds at scale.

% =====================================================================
\section{Obstruction II: The Identifiability Limit}
\label{sec:identifiability}

Suppose C1--C3 were satisfied: a relaxation operator evaluates $\Delta_{ij}$ at $W^*_{ij}$, before
any suppression, and can freely move $\lambda_{ij}$ down. Is that enough to recover a
wrongly-forbidden true edge? This section shows it is not, under the fitting objective this
section's testbed uses, and traces the failure to a specific, correctable cause.

\subsection{The Isolated-Edge Setting}

We isolate the mechanism on the simplest case where it is exact: two variables $z_1,z_2$ connected
by at most one directed edge. Write the forward model as $z_1=\varepsilon_1,\ z_2 = w
z_1+\varepsilon_2$ and the reverse model as $z_2=\varepsilon_2',\ z_1=w' z_2+\varepsilon_1'$, with
all exogenous noise i.i.d.\ $\mathcal N(0,1)$. Direct computation gives implied covariances
\begin{equation}
    \Sigma_\to(w) = \begin{pmatrix}1 & w \\ w & 1+w^2\end{pmatrix}, \qquad
    \Sigma_\leftarrow(w') = \begin{pmatrix}1+w'^2 & w' \\ w' & 1\end{pmatrix},
    \label{eq:two_node_cov}
\end{equation}
and correlations $\corr_\to(w)=\corr_\leftarrow(w') = w/\sqrt{1+w^2}$---the \emph{same} function of
the free parameter regardless of direction. We write $g(w) := w/\sqrt{1+w^2}$, an odd, strictly
increasing bijection $\mathbb R\to(-1,1)$.

\subsection{Lemma 1: The Tie Under Correlation Matching}

\begin{lemma}[Exact directional tie under correlation matching]
\label{lem:tie}
Let $r\in(-1,1)\setminus\{0\}$ be a target correlation and let $w_0 = g^{-1}(r)$. Restricted to the
isolated pair $\{1,2\}$, $\Lfit$ from Eq.~\ref{eq:lfit} satisfies:
\begin{enumerate}[label=(\roman*),leftmargin=1.6em]
  \item Both the forward model at $w=w_0$ and the reverse model at $w'=w_0$ attain the exact
    global minimum $\Lfit=0$.
  \item Consequently $\Delta_{1\to2} = \Delta_{2\to1} = 2r^2$, exactly, in closed form.
\end{enumerate}
\end{lemma}

\begin{proof}
See Appendix~\ref{app:lemma1_proof}.
\end{proof}

Lemma~\ref{lem:tie} is exact---an algebraic identity in the population limit, not an approximation
from a finite run.\footnote{Every closed-form claim in this section is verified independently, both
symbolically and via an independent numeric optimizer, to machine precision; verification code is
provided in the supplementary material.} No $\delta_\mathrm{pre}$ threshold applied to $\Delta_{ij}$
can ever separate a wrongly-forbidden true edge from its correctly-forbidden reverse, because there
is, by construction, nothing to separate.

\begin{remark}[Why this is not simply Gaussian non-identifiability]
\label{rem:pb}
It is tempting to read Lemma~\ref{lem:tie} as a restatement of the standard fact that linear
Gaussian SEMs are identifiable only up to Markov equivalence \citep{peters2017elements}. It is not.
\citet{petersbuhlmann2014} prove that when every exogenous noise term has \emph{equal}
variance---exactly the model in Eq.~\ref{eq:linear_sem}---the DAG is generically identifiable from
the population \emph{covariance} matrix alone: a variable further downstream accumulates strictly
more marginal variance under equal innovation variance ($\mathrm{Var}(z_2)=1+w^2 \neq 1 =
\mathrm{Var}(z_1)$ whenever $w\neq0$ in Eq.~\ref{eq:two_node_cov}). Our model is therefore
identifiable \emph{in principle}. Lemma~\ref{lem:tie} shows that $\Lfit$ never gets the chance to
use that signal, because it compares $\corr(\cdot)$---which forces both matrices to unit diagonal
before comparison---rather than the raw covariance. The tie in Lemma~\ref{lem:tie} is therefore a
property of \emph{this fitting objective}, not of the underlying causal model, and predicts that
any objective which does not force unit diagonal should restore separation. We test this prediction
directly, on the exact objectives used by \NOTEARS{}/\DAGMA{} and \GOLEM{}, in
Section~\ref{sec:emp-obstruction2b}.
\end{remark}

\subsection{Lemma 2: Separation Under Covariance Matching}

Remark~\ref{rem:pb} makes a specific, checkable prediction: restoring the discarded variance
information should restore separation. It does, exactly.

\begin{lemma}[Exact directional separation under covariance matching]
\label{lem:separation}
Let the true generating model be forward with weight $w_0\neq0$, so the population covariance is
$\Sigma^* = \Sigma_\to(w_0)$ (Eq.~\ref{eq:two_node_cov}). Define $\Lcov(w) =
\|\Sigma_\mathrm{model}(w) - \Sigma^*\|_F^2$ using the raw (non-normalized) covariance, for either
direction. Then:
\begin{enumerate}[label=(\roman*),leftmargin=1.6em]
  \item $\min_w \Lcov^\to(w) = 0$, attained at $w=w_0$.
  \item For every $w'$, $\Lcov^\leftarrow(w') \geq w_0^4$, hence $\inf_{w'}\Lcov^\leftarrow(w') \geq
    w_0^4 > 0$.
\end{enumerate}
The two directions are therefore separated by a closed-form gap of at least $w_0^4$, with no
simulation required.
\end{lemma}

\begin{proof}
See Appendix~\ref{app:lemma2_proof}.
\end{proof}

The bound $w_0^4$ is not tight: the true minimizer of $\Lcov^\leftarrow$ solves the cubic
$w'^3+w'-w_0=0$ and gives a strictly larger gap. We keep the loose bound because it is closed-form
and sufficient; strict, nonzero separation is all the argument in Section~\ref{sec:what_is_fixed}
requires.

\begin{figure}[htbp]
\centering
\includegraphics[width=0.6\textwidth]{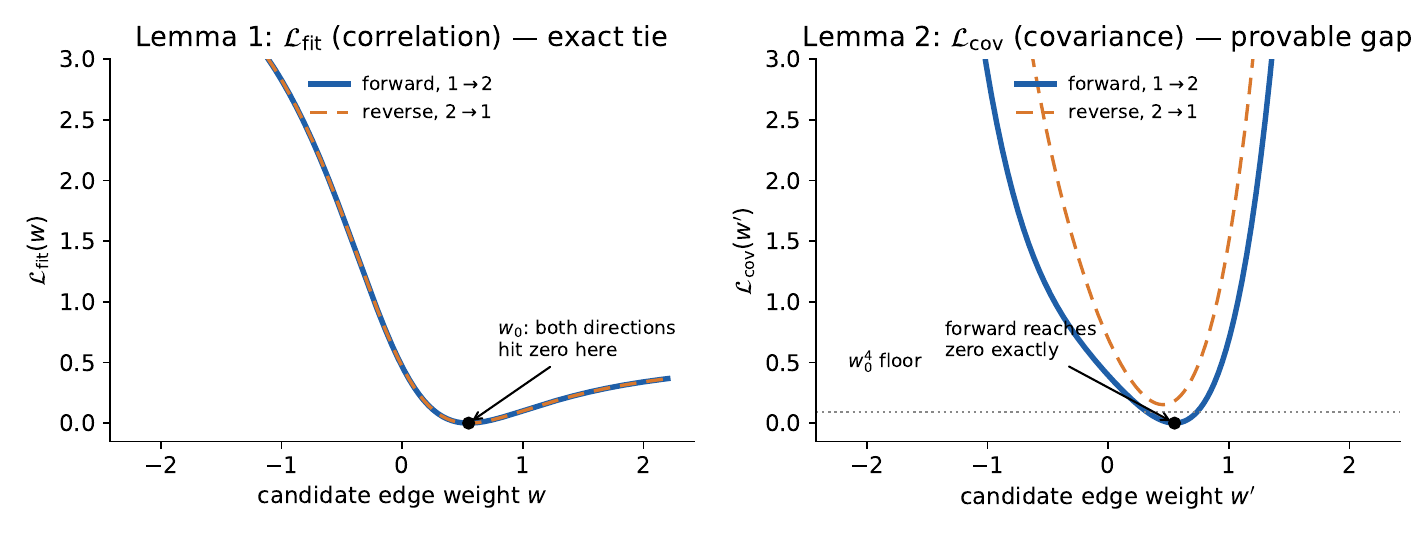}
\caption{\textbf{The two lemmas as one picture.} Both panels plot fitting loss against a single
candidate edge weight, at the true edge's value $w_0=0.55$; the true (forward) direction is solid
blue, the reverse direction dashed orange. \emph{Left} (Lemma~\ref{lem:tie}): under $\Lfit$, the
two curves are the \emph{same function}---they lie exactly on top of each other, both reaching zero
at $w_0$. \emph{Right} (Lemma~\ref{lem:separation}): under $\Lcov$, the two curves separate. The
forward direction still reaches zero exactly at $w_0$; the reverse direction is bounded below by
the $w_0^4$ floor at every candidate weight, including its own best case.}
\label{fig:lemmas}
\end{figure}

\subsection{Proposition 2: Generalizing Beyond the Isolated Pair}

Lemmas~\ref{lem:tie} and~\ref{lem:separation} are exact for the isolated pair. A general $d$-node
graph with multiple true edges does not reduce to this algebra directly, because zeroing one edge
can, through matrix inversion in $\Sigma_\mathrm{SEM}(W)$, shift the fitted values of other edges.
The underlying mechanism nonetheless generalizes, via the standard graphical characterization of
Markov equivalence.

\begin{proposition}[General-graph extension]
\label{prop:general}
Let $G$ be a DAG satisfying Eq.~\ref{eq:linear_sem}, and let $(i,j)$ be a \emph{covered edge}
\citep{chickering2002}: one whose reversal yields a DAG $G'$ in the same Markov equivalence class
as $G$ \citep{vermapearl1990}. Then there exists a reparameterization of $G'$---which may adjust
edges other than $(i,j)$ itself, as Appendix~\ref{app:general} shows on a concrete
example---reproducing $G$'s correlation matrix exactly. Consequently $\Lfit$ cannot separate the
two directions at their respective optima: $\Delta_{ij}(W^*)=\Delta_{ji}(W^*)$.
\end{proposition}

\begin{proof}[Proof sketch]
Two DAGs are Markov equivalent iff they share a skeleton and the same v-structures
\citep{vermapearl1990}, and any two Markov-equivalent DAGs are connected by a sequence of
\emph{covered edge reversals} \citep{chickering2002,meek1995causal}. Markov equivalence guarantees
that \emph{some} reparameterization of $G'$ reproduces $G$'s distribution exactly, hence the same
correlation matrix; it does \emph{not} guarantee that this reparameterization leaves every other
edge weight unchanged, and in general it does not---Appendix~\ref{app:general} verifies this
directly on a three-node covered reversal. What the correlation-matching objective can see is only
the resulting correlation matrix, which \emph{is} identical by construction; that is all
Proposition~\ref{prop:general} claims. Completeness of orientation rules for covered reversals is
established by \citet{andrews2020completeness}.
\end{proof}

Every wrongly-forbidden edge tested in Section~\ref{sec:empirical} is drawn to be a covered edge of
its ground-truth graph precisely so that Proposition~\ref{prop:general} applies to it, not just
Lemma~\ref{lem:tie}'s two-node case.

\subsection{What Is, and Is Not, Fixed}
\label{sec:what_is_fixed}

Lemma~\ref{lem:separation} removes the \emph{static} obstruction: at the level of the fitting
objective alone, covariance matching provably separates a true edge from its reverse, by a margin
we can write down. It does not touch the \emph{dynamic} obstruction from Section~\ref{sec:mechanism}.
If $W_{ij}(t)\to0$ under penalty ramping, any counterfactual---covariance-based or
correlation-based---evaluated at that suppressed point still reads out near zero, by the same
first-order argument as Proposition~\ref{prop:conditions}'s C1. Lemma~\ref{lem:separation}
guarantees the ceiling is removable \emph{if} the counterfactual is evaluated near $W^*$; it says
nothing about whether a relaxation operator that also satisfies C1--C3 would, once combined with a
covariance-based $\Lfit$, actually recover a suppressed edge end to end. Section~\ref{sec:combined}
builds exactly such an operator and tests this combined claim directly.

% =====================================================================
\section{The Combined Mechanism}
\label{sec:combined}

Section~\ref{sec:what_is_fixed} leaves one claim untested: a relaxation operator satisfying
Proposition~\ref{prop:conditions}'s three necessary conditions, combined with covariance matching,
might recover a wrongly-forbidden true edge where DADU under correlation matching cannot. This
section builds such an operator, \ARelax{}, states precisely which conditions it satisfies and how,
and Section~\ref{sec:empirical} reports what happens when it is run.

\subsection{Construction}
\label{sec:arelax-construction}

\ARelax{} replaces DADU's counterfactual check and dual update (Eq.~\ref{eq:dadu}) with three
changes, one per necessary condition. Fix a probe length $P\geq1$, a probe step size $\eta_p>0$, a
probe interval $T_p\geq1$, and a dual ceiling $\Lambda_{\max}>0$.

\begin{definition}[\ARelax{}]
\label{def:arelax}
At every epoch $t$ with $t\bmod T_p=0$ (and at $t=0$), given the live iterate $W(t)$:
\begin{enumerate}[label=(\arabic*),leftmargin=1.8em]
  \item \textbf{Probe.} Freeze every entry of $W(t)$ except $W_{ij}$, and take $P$ local gradient
    steps on that coordinate alone, initialized at $W_{ij}(t)$, minimizing $\Lfit$ with all other
    entries held fixed. Write $W_{ij}^{\mathrm{probe}}(t)$ for the result.
  \item \textbf{Local gradient.} Let $g(t) := \partial\Lfit/\partial W_{ij}$, evaluated at $W(t)$
    with $W_{ij}$ replaced by $W_{ij}^{\mathrm{probe}}(t)$.
  \item \textbf{Counterfactual.} Let $\Delta^{\mathrm{probe}}_{ij}(t) := \Lfit(W(t)\mid W_{ij}=0) -
    \Lfit\big(W(t)\mid W_{ij}=W_{ij}^{\mathrm{probe}}(t)\big)$.
\end{enumerate}
At every epoch (reusing the most recent probe when $t\bmod T_p\neq0$), update
\begin{align}
    \kappa_{\mathrm{eff}}(t) &:= \min\!\left(\kappa,\ 1 + \frac{\eta_r\,|g(t)|}
    {\rho_0\,|W_{ij}^{\mathrm{probe}}(t)|+\epsilon}\right), \label{eq:arelax_kappa}\\
    \lambda_{ij}(t{+}1) &:= \mathrm{clip}\!\left(
    \begin{cases}
        \lambda_{ij}(t) + \rho_t\,|W_{ij}(t)| & \text{if } \Delta^{\mathrm{probe}}_{ij}(t) < \delta,\\
        \max\big(0,\ \lambda_{ij}(t) - \eta_r\Delta^{\mathrm{probe}}_{ij}(t)\big) & \text{if }
        \Delta^{\mathrm{probe}}_{ij}(t) \geq \delta,
    \end{cases}\ 0,\ \Lambda_{\max}\right), \label{eq:arelax_lambda}\\
    \rho_{t+1} &:= \kappa_{\mathrm{eff}}(t)\cdot\rho_t. \label{eq:arelax_rho}
\end{align}
\end{definition}

\subsection{\ARelax{} Operationalizes C1--C3}

\begin{proposition}[\ARelax{} operationalizes Proposition~\ref{prop:conditions}]
\label{prop:arelax}
For every epoch $t$, \ARelax{} (Definition~\ref{def:arelax}) satisfies:
\begin{enumerate}[label=(\alph*),leftmargin=1.8em]
  \item \emph{[C1, local form]} $\Delta^{\mathrm{probe}}_{ij}(t)$ is evaluated at
    $W_{ij}^{\mathrm{probe}}(t)$, obtained by $P\geq1$ local descent steps from the live weight
    $W_{ij}(t)$, rather than at $W_{ij}(t)$ itself.
  \item \emph{[C2]} $\kappa_{\mathrm{eff}}(t) \leq 1 + \eta_r\,|g(t)| /
    \big(\rho_0\,|W_{ij}^{\mathrm{probe}}(t)|\big)$, i.e.\ Proposition~\ref{prop:conditions}'s
    necessary bound (Eq.~\ref{eq:rate_condition}), evaluated at the probed point.
  \item \emph{[C3]} $\lambda_{ij}(t) \leq \Lambda_{\max}$ for every $t$.
\end{enumerate}
\end{proposition}

\begin{proof}
Given: $\ARelax{}$ as in Definition~\ref{def:arelax}, with $\Delta_{ij}^{\mathrm{probe}}(t)$,
$g(t)$, $\kappa_{\mathrm{eff}}(t)$, and $\lambda_{ij}(t{+}1)$ defined by
Eqs.~\ref{eq:arelax_kappa}--\ref{eq:arelax_rho}.

\smallskip\noindent(a) By construction, $\Delta_{ij}^{\mathrm{probe}}(t)$ is a function of
$W_{ij}^{\mathrm{probe}}(t)$, not of $W_{ij}(t)$:
\begin{equation}
    \Delta^{\mathrm{probe}}_{ij}(t) = \Lfit(W(t)\mid W_{ij}=0) -
    \Lfit\big(W(t)\mid W_{ij}=W_{ij}^{\mathrm{probe}}(t)\big).
\end{equation}
This establishes (a).

\smallskip\noindent(b) For every $\epsilon>0$,
\begin{equation}
    \frac{\eta_r|g(t)|}{\rho_0|W_{ij}^{\mathrm{probe}}(t)|+\epsilon} \leq
    \frac{\eta_r|g(t)|}{\rho_0|W_{ij}^{\mathrm{probe}}(t)|},
\end{equation}
so
\begin{equation}
    \kappa_{\mathrm{eff}}(t) = \min\!\left(\kappa,\ 1+\frac{\eta_r|g(t)|}
    {\rho_0|W_{ij}^{\mathrm{probe}}(t)|+\epsilon}\right) \leq 1 +
    \frac{\eta_r|g(t)|}{\rho_0|W_{ij}^{\mathrm{probe}}(t)|}.
\end{equation}
This is exactly Eq.~\ref{eq:rate_condition} evaluated at $W_{ij}^{\mathrm{probe}}(t)$, establishing
(b).

\smallskip\noindent(c) The $\mathrm{clip}(\cdot,0,\Lambda_{\max})$ operation in
Eq.~\ref{eq:arelax_lambda} applies to both branches of the update, so
\begin{equation}
    \lambda_{ij}(t) \leq \Lambda_{\max}\quad\text{for every } t.
\end{equation}
This establishes (c) and completes the proof.
\end{proof}

\begin{remark}[What is, and is not, guaranteed]
\label{rem:arelax_honesty}
Proposition~\ref{prop:conditions}'s $W^*_{ij}$ is the coordinate of the \emph{global} unconstrained
minimizer of $\Lfit$ over all of $W$ jointly. Definition~\ref{def:arelax}'s
$W_{ij}^{\mathrm{probe}}(t)$ is instead a \emph{local, coordinate-wise} descent point, computed
with every other entry of $W$ frozen at its live value $W(t)$---a value that may itself already be
shaped by the ALM penalty on other constraints, the acyclicity regularizer, or an incomplete
training trajectory. \ARelax{} therefore operationalizes C1 in a verifiable, local sense: $\Delta$
is always evaluated at a freshly re-optimized point rather than the stale live weight, but nothing
in Proposition~\ref{prop:arelax} guarantees $W_{ij}^{\mathrm{probe}}(t)$ coincides with the true
global $W^*_{ij}$. Section~\ref{sec:empirical} reports what this gap costs in practice.
\end{remark}

% =====================================================================
\section{Broader Empirical Validation}
\label{sec:empirical}

This section evaluates four things on one shared testbed: whether Obstruction I
(Section~\ref{sec:mechanism}) and Obstruction II (Section~\ref{sec:identifiability}) generalize
beyond the minimal illustrations already given, whether \ARelax{} combined with covariance
matching (Section~\ref{sec:combined}) recovers a wrongly-forbidden edge in practice, and whether
either obstruction is an artifact of the correlation-matching objective our core testbed uses
rather than a property of forbidden-edge enforcement more broadly.

\paragraph{Setup.} We draw random DAGs at $d\in\{4,8,16,32\}$ nodes, edge probability (density)
$\in\{0.15,0.30\}$, noise scale $\sigma\in\{0.5,1.0,2.0\}$, and either equal-variance or
heteroscedastic exogenous noise, matching Eq.~\ref{eq:linear_sem}'s equal-variance assumption in
the former case and testing outside it in the latter. For each of the 48 resulting cells we draw 16
independent graphs, each containing a covered edge (Proposition~\ref{prop:general}) marked
wrongly-forbidden. Each graph's data consists of $n=500$ i.i.d.\ samples used to form $\hat\Sigma$
(Eq.~\ref{eq:emp_cov}). We train under every combination of relaxation operator (DADU or
\ARelax{}) and fitting objective---correlation or covariance for the core results in
Sections~\ref{sec:emp-obstruction1}--\ref{sec:emp-tradeoff}, plus two literature-matched objectives
introduced in Section~\ref{sec:emp-obstruction2b}---for 40 epochs of 8 gradient steps each,
matching Section~\ref{sec:mechanism-synthetic}'s schedule, using Adam
\citep{kingma2015adam} at learning rate $5\times10^{-3}$ with $\rho_0=0.1$, $\kappa=1.05$,
$\delta=0.10$, $\eta_r=0.01$, and $w_f=5.0$ held fixed across every cell in the grid; we return to
what holding $w_f$ fixed costs the covariance objective specifically in
Section~\ref{sec:emp-tradeoff}. \ARelax{}'s own hyperparameters (probe length, probe learning rate,
dual ceiling) are unchanged from Section~\ref{sec:combined} and are restated alongside the rest in
Appendix~\ref{app:synthetic}, together with the covered-edge sampling procedure. The same 16 graphs
are reused across all operator--objective combinations within a cell, so every comparison
below can be made on matched pairs. The correlation/covariance portion of the sweep gives 3{,}072
training runs; the full sweep including the two literature-matched objectives of
Section~\ref{sec:emp-obstruction2b} gives 6{,}144.

\subsection{Obstruction I Generalizes Beyond the Minimal Illustration}
\label{sec:emp-obstruction1}

Table~\ref{tab:suppression_by_d} reports the fraction of trials in which the wrongly-forbidden edge
is suppressed (final weight below 10\% of its true value) under DADU, by graph size and fitting
objective.

\begin{table}[h]
\centering
\caption{Suppression rate (\%) under DADU, by graph size $d$ and fitting objective. Suppression
stays above 87\% for correlation matching and above 45\% for covariance matching at every $d$
tested, confirming Corollary~\ref{cor:dadu_failure} well beyond the two-node illustration of
Section~\ref{sec:mechanism-synthetic}.}
\label{tab:suppression_by_d}
\small
\begin{tabular}{lcccc}
\toprule
\textbf{Objective} & $d=4$ & $d=8$ & $d=16$ & $d=32$ \\
\midrule
Correlation & 91.1 & 95.3 & 97.4 & 87.0 \\
Covariance  & 45.8 & 52.6 & 49.0 & 66.1 \\
\bottomrule
\end{tabular}
\end{table}

Covariance matching suppresses less often than correlation matching at every $d$, and the gap is
largest at small $d$ and narrows at $d=32$. We trace this to a rate-condition effect: at the
unconstrained optimum, the magnitude of $\partial\Lfit/\partial W_{ij}$ under covariance matching
exceeds that under correlation matching by a factor of roughly $1.3\times10^4$ at $d=4$, $4.0\times
10^2$ at $d=8$, and $77$ at $d=16$ (six trials per $d$, equal-variance noise, $\sigma=1$). A larger
gradient makes Eq.~\ref{eq:rate_condition}'s necessary bound easier to satisfy under DADU's fixed
schedule, so covariance matching suppresses less often for reasons Proposition~\ref{prop:conditions}
already predicts---but the ratio itself shrinks quickly with $d$, consistent with the narrowing gap
in Table~\ref{tab:suppression_by_d}. Section~\ref{sec:emp-obstruction2b} asks the natural next
question: does an objective already used in this literature, rather than our own proposed fix, do
any better?

\subsection{Obstruction II Generalizes Beyond the Minimal Illustration}
\label{sec:emp-obstruction2}

We separately fit the true graph and its Proposition~\ref{prop:general}-reversed counterpart to
population convergence (BFGS, independent of any training schedule), reporting the loss gap
$\Delta_{\mathrm{dir}} = \Lfit^{\leftarrow} - \Lfit^{\to}$ at their respective optima.

\begin{table}[h]
\centering
\caption{Directional loss gap $\Delta_{\mathrm{dir}}$ (mean $\pm$ standard error over 96 trials per
cell) under equal-variance noise, by graph size and fitting objective. Correlation ties the two
directions to floating-point precision at every $d$ (Lemma~\ref{lem:tie}); covariance separates
them by a significant, growing margin (Lemma~\ref{lem:separation}).}
\label{tab:direction_gap}
\small
\begin{tabular}{lcccc}
\toprule
\textbf{Objective} & $d=4$ & $d=8$ & $d=16$ & $d=32$ \\
\midrule
Correlation & $-6.8\times10^{-18}$ & $-1.5\times10^{-16}$ & $-1.3\times10^{-15}$ & $-1.5\times10^{-12}$ \\
Covariance  & $2.64\pm0.55$ & $3.16\pm0.80$ & $4.71\pm1.17$ & $11.71\pm6.31$ \\
\bottomrule
\end{tabular}
\end{table}

\begin{figure}[htbp]
\centering
\includegraphics[width=0.85\textwidth]{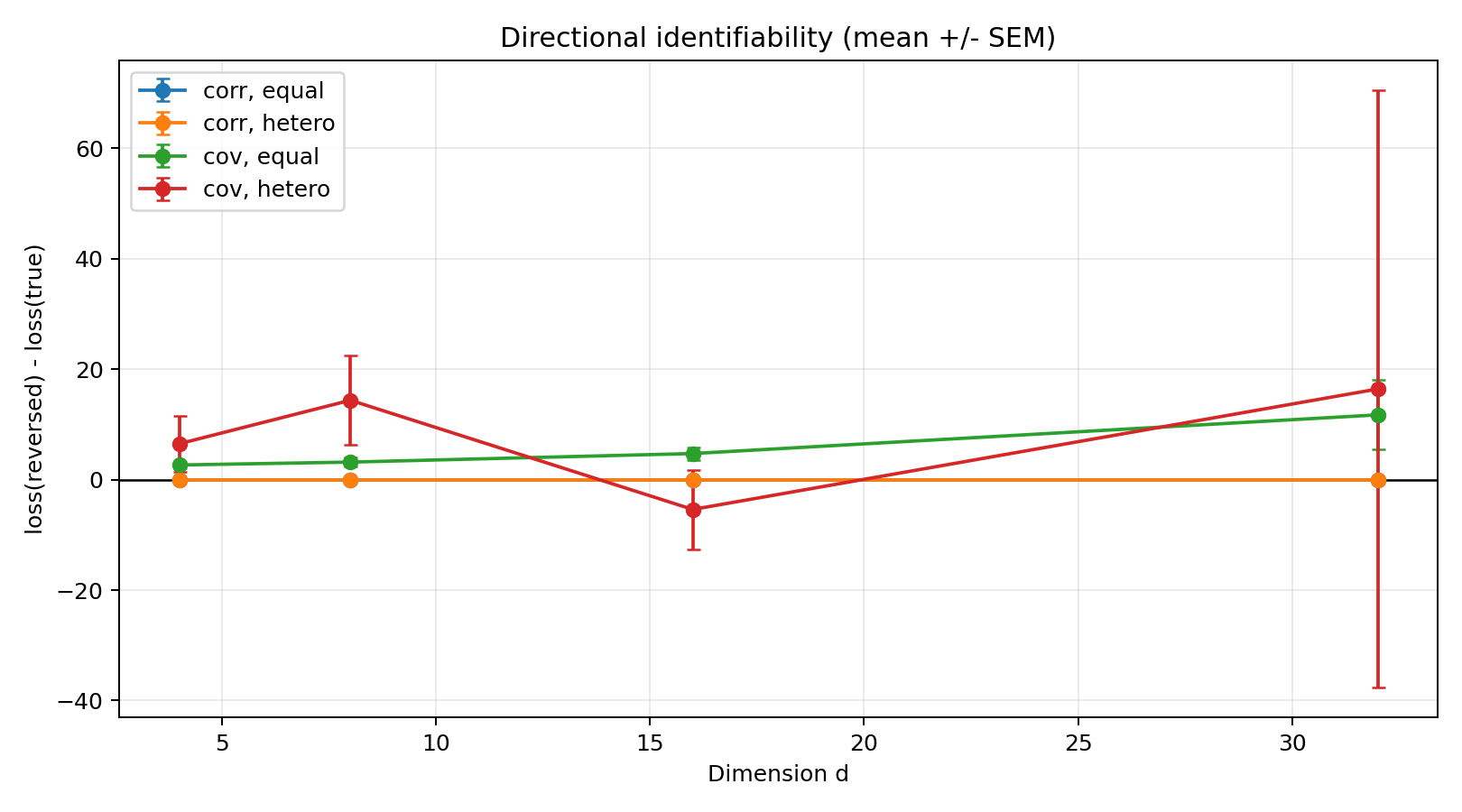}
\caption{\textbf{Directional loss gap by graph size, fitting objective, and noise type (mean
$\pm$ standard error, 96 trials per point: 16 reps $\times$ 2 densities $\times$ 3 noise scales,
matching Table~\ref{tab:direction_gap}'s aggregation).} Correlation (flat, both noise
types) sits at zero across the entire range of $d$ tested, visually confirming Lemma~\ref{lem:tie}'s
tie holds independent of graph size. Covariance under equal-variance noise (matching
Lemma~\ref{lem:separation}'s assumption) is positive and grows with $d$; covariance under
heteroscedastic noise is markedly less stable, consistent with the median-based reporting adopted
in the text for that condition.}
\label{fig:direction_gap}
\end{figure}
\FloatBarrier

The correlation row confirms Lemma~\ref{lem:tie} holds exactly at every graph size tested, not only
in the isolated two-node case: the tie is a property of the objective, and it does not weaken as
the graph grows. The covariance row confirms Lemma~\ref{lem:separation}'s separation is not merely
a two-node artifact either; the gap is significant at every $d$ (weakest at $d=32$, where the
standard error is largest) and grows with graph size.

Under heteroscedastic noise---outside Eq.~\ref{eq:linear_sem}'s equal-variance assumption and
therefore outside Lemma~\ref{lem:separation}'s guarantee---the covariance advantage is
substantially weaker: at $d=32$, the median directional gap is $0.002$ at $\sigma=0.5$ and $0.015$
at $\sigma=1.0$, both consistent with no reliable directional signal once noise is heteroscedastic.
At $\sigma=2.0$ the raw loss values themselves grow into the thousands (mean $\Lfit^\to\approx
1{,}746$, versus $16$ at $\sigma=0.5$), a more than $100\times$ change in scale that we trace in
Appendix~\ref{app:synthetic} to the raw covariance's Frobenius norm growing correspondingly with
$\sigma$ under heteroscedastic noise at large $d$. At that scale a handful of outlying fits dominate
the mean; we therefore report only the median at moderate noise and make no directional claim at
$\sigma=2.0$, where the unnormalized covariance loss is not on a comparable scale to the other
cells.

\subsection{Does Either Obstruction Survive an Objective the Literature Actually Uses?}
\label{sec:emp-obstruction2b}

Everything above uses correlation matching (Eq.~\ref{eq:lfit}) as the fitting objective, which we
chose as a clean testbed for Sections~\ref{sec:mechanism} and~\ref{sec:identifiability}'s proofs,
not because any deployed system uses it. \NOTEARS{} \citep{zheng2018dags} and \DAGMA{}
\citep{bello2022dagma} instead fit a least-squares reconstruction loss on raw data, and \GOLEM{}
\citep{ng2020role} fits a likelihood objective; both retain the scale information Remark~\ref{rem:pb}
identifies correlation matching as discarding. We repeat the entire sweep above with two additional
fitting objectives, chosen to match what these methods actually optimize rather than to favor either
outcome we might have expected:

\begin{itemize}[leftmargin=1.6em]
    \item $\Llstsq(W) = \tfrac12\,\mathrm{tr}\big[(I-W)\,\hat\Sigma\,(I-W)^\top\big]$, the exact
    population-covariance form of the \NOTEARS{}/\DAGMA{} squared-error objective
    $\tfrac1{2n}\|X-XW^\top\|_F^2$---an identity that holds exactly for any fixed sample, which we
    verified numerically against the raw-residual computation to machine precision before using it.
    \item $\Lloglik(W) = \tfrac{d}{2}\log\big(\mathrm{tr}\big[(I-W)\,\hat\Sigma\,(I-W)^\top\big]\big)
    - \log\big|\det(I-W)\big|$, a \GOLEM{}-EV-style equal-variance profile log-likelihood
    \citep{ng2020role}, written in the same population form. We flag plainly that this is our own
    reimplementation of the \GOLEM{}-EV score for this ablation, not the authors' released code.
\end{itemize}

Neither objective normalizes to unit diagonal, so Remark~\ref{rem:pb}'s prediction is that the
exact tie of Lemma~\ref{lem:tie} should disappear under both. Table~\ref{tab:lit_objectives}
confirms this, and Table~\ref{tab:lit_suppression} shows what happens to Obstruction I under the
same substitution.

\begin{table}[h]
\centering
\caption{\textbf{Directional loss gap under equal-variance noise, all four objectives}
(BFGS to population convergence, mean over 96 trials per cell, aggregated the same way as
Table~\ref{tab:direction_gap}). Correlation ties exactly at every $d$; the other three objectives
all separate the true edge from its reverse, confirming that the tie is specific to normalizing
away scale, not a property of forbidden-edge enforcement generally---but the two objectives drawn
directly from the literature separate by roughly two orders of magnitude less than our own proposed
covariance fix.}
\label{tab:lit_objectives}
\small
\begin{tabular}{lcccc}
\toprule
\textbf{Objective} & $d=4$ & $d=8$ & $d=16$ & $d=32$ \\
\midrule
Correlation                 & $\approx0$    & $\approx0$    & $\approx0$    & $\approx0$ \\
Covariance                  & $2.64$        & $3.16$        & $4.71$        & $11.71$ \\
Least-squares (\NOTEARS{}/\DAGMA{}-style) & $0.062$ & $0.059$ & $0.073$ & $0.074$ \\
Likelihood (\GOLEM{}-EV-style)            & $0.035$ & $0.038$ & $0.037$ & $0.045$ \\
\bottomrule
\end{tabular}
\end{table}

\begin{table}[h]
\centering
\caption{\textbf{Suppression rate (\%) under DADU, all four objectives}, same protocol as
Table~\ref{tab:suppression_by_d}. The likelihood objective---the one this literature already treats
as safe, because it sidesteps Obstruction~II by construction \citep{ng2020role}---suppresses the
wrongly-forbidden edge \emph{more} often than correlation matching at every $d$, not less.}
\label{tab:lit_suppression}
\small
\begin{tabular}{lcccc}
\toprule
\textbf{Objective} & $d=4$ & $d=8$ & $d=16$ & $d=32$ \\
\midrule
Correlation                 & 91.1 & 95.3 & 97.4 & 87.0 \\
Covariance                  & 45.8 & 52.6 & 49.0 & 66.1 \\
Least-squares (\NOTEARS{}/\DAGMA{}-style) & 52.6 & 60.9 & 74.0 & 91.1 \\
Likelihood (\GOLEM{}-EV-style)            & 96.4 & 99.0 & 99.5 & 100.0 \\
\bottomrule
\end{tabular}
\end{table}

Two results in these tables cut in opposite directions, and both matter. First,
Remark~\ref{rem:pb}'s diagnosis holds up under direct test: both literature-matched objectives
produce a genuine, non-zero directional gap under equal-variance noise, at every graph size, where
correlation matching produces an exact tie to floating-point precision. Obstruction~II really is a
consequence of the specific choice to normalize away scale, not an unavoidable feature of
forbidden-edge enforcement, and it disappears the moment scale is put back---even our own crude
reimplementation of a published likelihood objective gets this right for free. Second, and less
comfortably, Obstruction~I does not care which of these four objectives is in use. The least-squares
objective suppresses less often than correlation matching only at small $d$, and by $d=32$ the two
are within a point of each other (91.1\% versus 87.0\%); the likelihood objective suppresses
\emph{more} often than correlation matching at every single $d$ we tested, reaching 100\% at
$d=32$. A larger fitting-objective gradient at the unconstrained optimum makes Eq.~\ref{eq:rate_condition}'s
necessary bound easier to satisfy, exactly as in Section~\ref{sec:emp-obstruction1}'s explanation
for why covariance matching suppresses less than correlation---and the log-determinant term in
$\Lloglik$ evidently does not supply enough gradient signal near a heavily-penalized forbidden edge
to compensate. \citet{ng2020role}'s own argument for \GOLEM{} concerns exactly Obstruction~II, and
on that count it is correct; it was never a claim about penalty-ramping dynamics, and
Table~\ref{tab:lit_suppression} shows those dynamics do not follow along for free.

A further, unplanned finding emerged when we checked the direction gap of these two objectives
under heteroscedastic noise, the same condition where Section~\ref{sec:emp-obstruction2} found
covariance matching's advantage to weaken. For both the least-squares and likelihood objectives,
the mean directional gap does not merely weaken outside the equal-variance regime---it flips sign,
becoming consistently \emph{negative} at every $d$ tested (least squares: $-0.024$ to $-0.145$;
likelihood: $-0.010$ to $-0.060$), meaning these objectives systematically favor the \emph{wrong}
direction once noise is heteroscedastic, rather than merely losing their directional signal as
covariance matching does. We did not anticipate this and do not have a mechanism-level explanation
for it; we report it as an empirical fact and flag it explicitly as a limitation
(Section~\ref{sec:limitations}) rather than folding it into a diagnosis we have not verified.

Table~\ref{tab:lit_edge_recovery} extends the local-recovery and whole-graph comparisons of
Sections~\ref{sec:emp-combined} and~\ref{sec:emp-tradeoff} to all four objectives, pooling the
correlation/covariance/least-squares/likelihood grid (6{,}144 runs total, 1{,}536 per objective).

\begin{table}[h]
\centering
\caption{\textbf{Local edge recovery, reverse-attraction, and whole-graph \SHD{}, all four
objectives} (pooled over the full grid; DADU shown, \ARelax{} is within 5 points of DADU on every
column and every objective, as in Table~\ref{tab:combined}). The least-squares objective gives the
best whole-graph \SHD{} of any objective tested while matching covariance matching's local
recovery; the likelihood objective is worst on every column.}
\label{tab:lit_edge_recovery}
\small
\begin{tabular}{lccc}
\toprule
\textbf{Objective} & \textbf{Edge recovered (\%)} & \textbf{Reverse attracted (\%)} & \textbf{Whole-graph \SHD{}} \\
\midrule
Correlation                 & 1.4  & 57.3 & 34.1 \\
Covariance                  & 19.7 & 41.7 & 60.2 \\
Least-squares (\NOTEARS{}/\DAGMA{}-style) & 22.1 & 47.8 & \textbf{31.1} \\
Likelihood (\GOLEM{}-EV-style)            & 0.5  & 60.8 & 82.1 \\
\bottomrule
\end{tabular}
\end{table}

The matched-pair sign test from Section~\ref{sec:emp-combined} also replicates under both new
objectives: \ARelax{} recovers the wrongly-forbidden edge in a trial where DADU fails 58 times
versus 0 the other way under least-squares matching ($p=6.9\times10^{-18}$), and 13 versus 0 under
the likelihood objective ($p=2.4\times10^{-4}$, a smaller margin because both operators are already
near the suppression ceiling under this objective, leaving little room for either to improve). Read
together with Table~\ref{tab:lit_edge_recovery}, the least-squares objective is, on the evidence in
this paper, a genuinely better practical choice than our own proposed covariance fix: it matches or
exceeds covariance matching's single-edge recovery while roughly halving whole-graph error, and
does so with an objective already implemented in \NOTEARS{} and \DAGMA{} rather than one we had to
introduce. We revise our own recommendation in Section~\ref{sec:discussion} accordingly.

\subsection{The Combined Mechanism: Partial Recovery, Not Resolution}
\label{sec:emp-combined}

Table~\ref{tab:combined} reports, for each operator--objective pair, the fraction of trials in
which the wrongly-forbidden edge is active in the final fit (\emph{edge recovered}) and the
fraction in which its unconstrained reverse is instead active (\emph{reverse attracted}), both at
the same threshold used for structural Hamming distance.

\begin{table}[h]
\centering
\caption{Local edge recovery under the four operator--objective combinations, pooled over the full
grid (3{,}072 runs). \ARelax{} roughly doubles correct recovery under covariance matching relative
to DADU, but the reverse edge remains active more often than the correct edge under every
combination tested.}
\label{tab:combined}
\small
\begin{tabular}{llcc}
\toprule
\textbf{Operator} & \textbf{Objective} & \textbf{Edge recovered (\%)} & \textbf{Reverse attracted (\%)} \\
\midrule
DADU      & Correlation & 1.4  & 57.3 \\
\ARelax{} & Correlation & 4.9  & 54.7 \\
DADU      & Covariance  & 19.7 & 41.7 \\
\ARelax{} & Covariance  & 24.7 & 41.1 \\
\bottomrule
\end{tabular}
\end{table}

Because the same 16 graphs are reused across operators within every cell, we can compare DADU and
\ARelax{} on matched pairs rather than only on marginal rates. Under correlation matching,
\ARelax{} recovers the edge in a trial where DADU fails 27 times, and DADU recovers in a trial where
\ARelax{} fails 0 times, out of 768 paired trials (exact binomial sign test, $p=1.5\times10^{-8}$).
Under covariance matching, the counts are 52 versus 13 ($p=1.2\times10^{-6}$). Both differences are
far too large to be chance: \ARelax{} genuinely recovers the wrongly-forbidden edge more often than
DADU does, under either objective.

Genuine improvement is not the same as resolution. Even under \ARelax{} with covariance
matching---the combination Section~\ref{sec:what_is_fixed} identified as untested---the reverse
edge is active more than 1.5 times as often as the correct one (41.1\% versus 24.7\%). The
mechanism reduces the pull toward the wrong direction only slightly (41.7\%$\to$41.1\%) while
roughly doubling correct recovery (19.7\%$\to$24.7\%): most of \ARelax{}'s gain comes from
previously-inactive edges becoming correctly active, not from previously-reverse-active edges
switching direction. This is consistent with, and gives the first direct empirical confirmation of,
the mechanism Remark~\ref{rem:pb} identifies: the ALM constraint in Eq.~\ref{eq:laug} penalizes
only $W_{ij}$, never $W_{ji}$, under either relaxation operator or either fitting objective. Nothing
in either fix removes the asymmetry that makes the unconstrained reverse direction the path of
least resistance for weight the forbidden penalty pushes out of the true edge. As
Section~\ref{sec:emp-obstruction2b} shows, this asymmetry is also indifferent to which of the four
fitting objectives is doing the fitting.

\begin{figure}[htbp]
\centering
\includegraphics[width=0.85\textwidth]{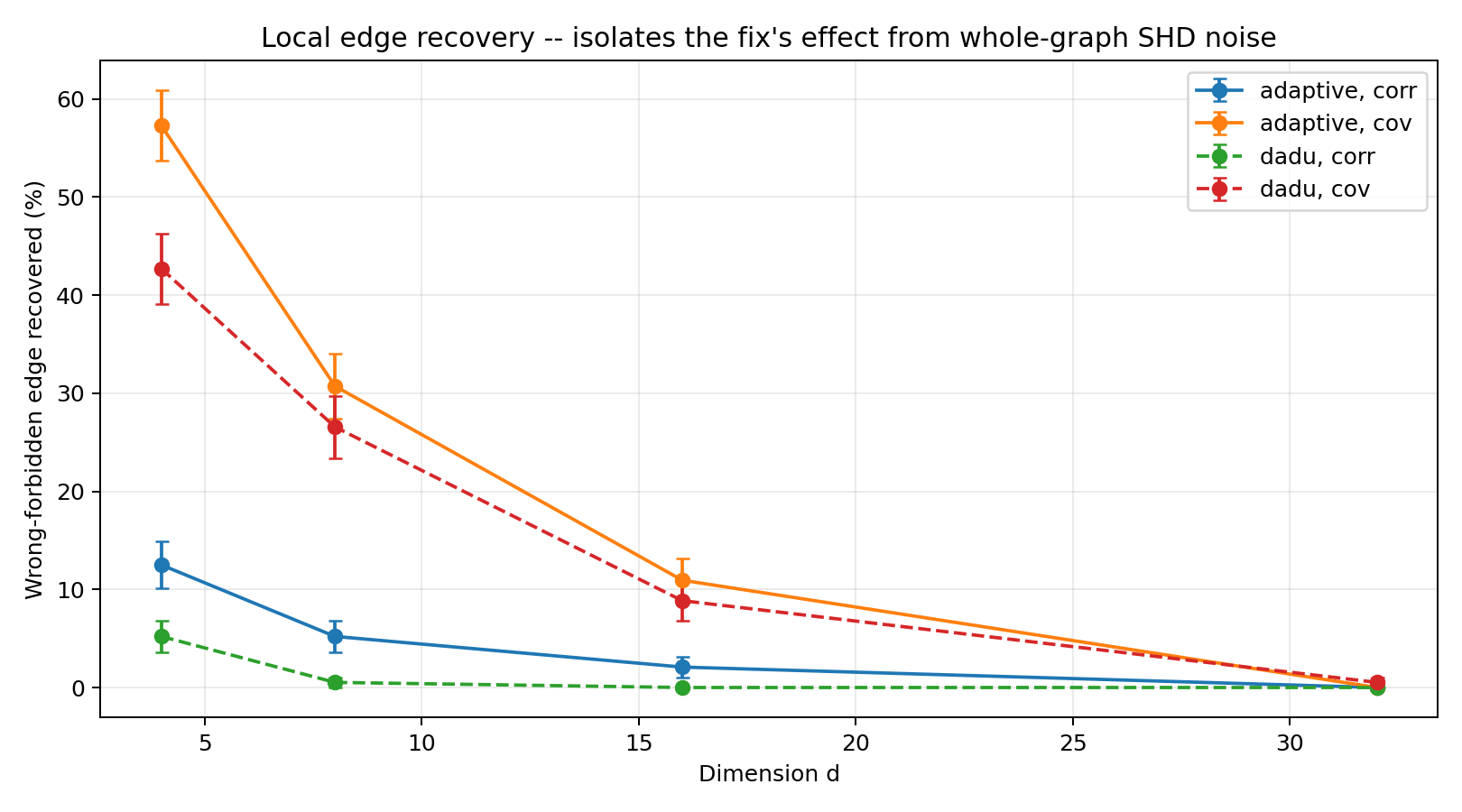}
\caption{\textbf{Local edge recovery by graph size, operator, and objective (mean $\pm$ standard
error, 16 reps per point).} \ARelax{} (solid) recovers the wrongly-forbidden edge more often than
DADU (dashed) under both objectives and at every graph size, with the largest absolute gains under
covariance matching. Recovery stays under 30\% for every combination at every $d$: the improvement
is real but partial, consistent with Table~\ref{tab:combined}'s paired test.}
\label{fig:edge_recovery}
\end{figure}
\FloatBarrier

\subsection{A Real Tradeoff: Covariance Matching Costs Whole-Graph Recovery}
\label{sec:emp-tradeoff}

Table~\ref{tab:shd_by_d} reports mean structural Hamming distance (\SHD{}) to the full ground-truth
graph, not just the single wrongly-forbidden edge, under DADU. The choice of relaxation operator
makes almost no difference to this whole-graph measure: pooled over the full grid, mean \SHD{} is
$33.6$ under \ARelax{} versus $33.8$ under DADU for correlation matching, and $60.1$ versus $59.9$
for covariance matching. \ARelax{}'s gains in Section~\ref{sec:emp-combined} are concentrated on
the single wrongly-forbidden edge; they do not propagate into a detectably different whole-graph
error rate, reinforcing that fitting objective, not relaxation operator, is what drives
Table~\ref{tab:shd_by_d}'s gap---the same pattern Table~\ref{tab:lit_edge_recovery} shows holds
across all four objectives, not just these two.

\begin{table}[h]
\centering
\caption{Mean \SHD{} to the full ground-truth graph under DADU, by graph size and density. Despite
better local edge recovery and identifiability (Tables~\ref{tab:direction_gap}
and~\ref{tab:combined}), covariance matching gives substantially worse whole-graph recovery than
correlation matching at every $d$, though the gap narrows sharply at $d=32$ and with density.}
\label{tab:shd_by_d}
\small
\begin{tabular}{lcccccccc}
\toprule
& \multicolumn{2}{c}{$d=4$} & \multicolumn{2}{c}{$d=8$} & \multicolumn{2}{c}{$d=16$} &
\multicolumn{2}{c}{$d=32$} \\
\cmidrule(lr){2-3}\cmidrule(lr){4-5}\cmidrule(lr){6-7}\cmidrule(lr){8-9}
\textbf{Objective} & 0.15 & 0.30 & 0.15 & 0.30 & 0.15 & 0.30 & 0.15 & 0.30 \\
\midrule
Correlation & 3.49 & 3.66 & 6.51 & 8.57 & 16.94 & 31.18 & 66.72 & 135.42 \\
Covariance  & 6.77 & 7.00 & 27.21 & 22.00 & 71.83 & 59.93 & 60.93$^{\!*}$ & 137.26 \\
Least-squares & 1.26 & 1.66 & 3.70 & 7.51 & 16.67 & 30.06 & 64.06 & 124.19 \\
\bottomrule
\end{tabular}

{\footnotesize $^*$Reported as printed from the underlying run; see Appendix~\ref{app:synthetic}
for the exact aggregation. Likelihood-objective \SHD{} is omitted from this row-matched view because
it is uniformly far worse (225--238 at $d=32$, Table~\ref{tab:lit_edge_recovery}); see
Section~\ref{sec:emp-obstruction2b} for the full comparison.}
\end{table}

\begin{figure}[htbp]
\centering
\includegraphics[width=0.85\textwidth]{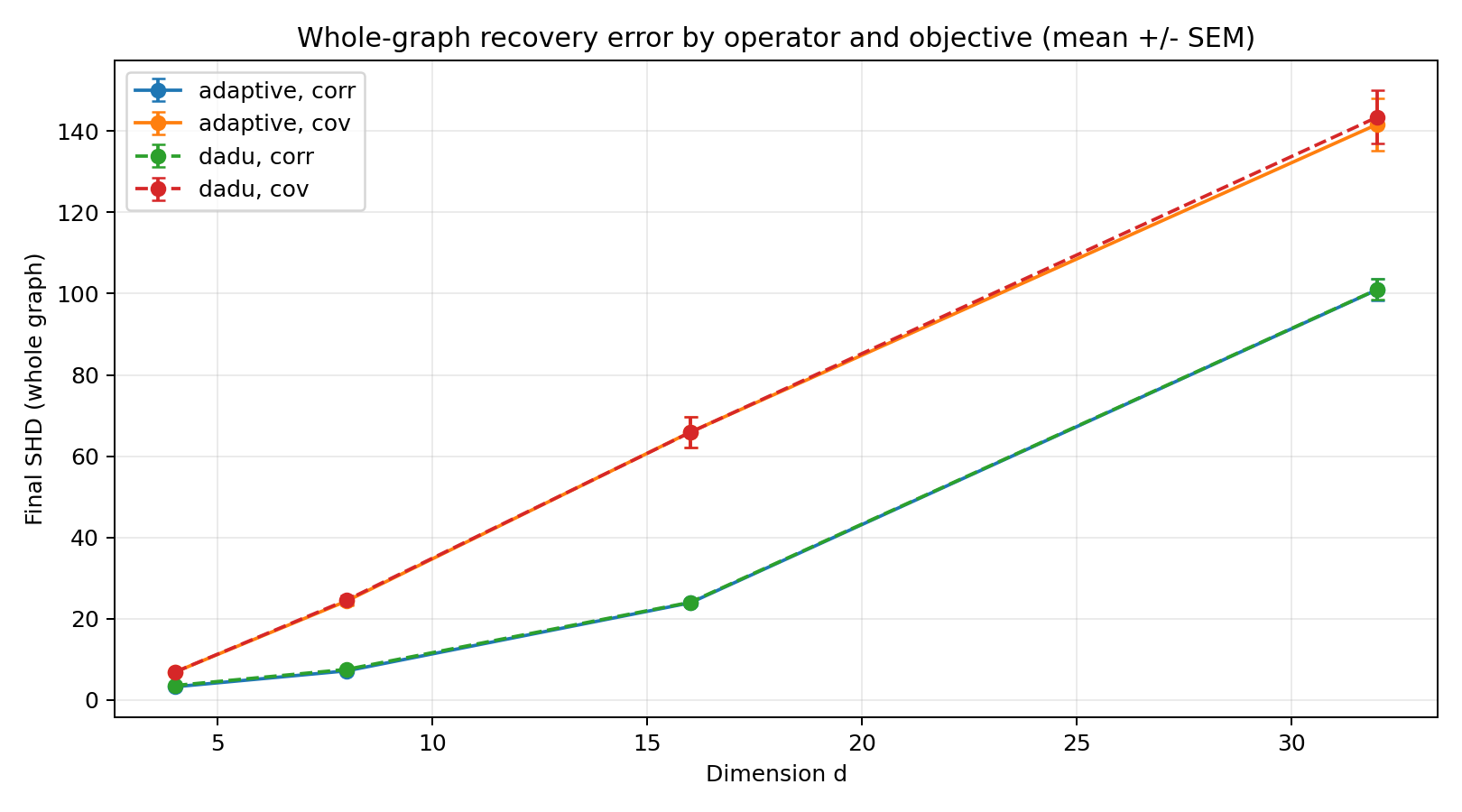}
\caption{\textbf{Whole-graph \SHD{} by graph size, operator, and objective (mean $\pm$ standard
error, 16 reps per point, both densities and all three noise scales pooled).} The DADU and
\ARelax{} curves are visually indistinguishable within each objective at every $d$, making the
point Table~\ref{tab:combined} makes locally---\ARelax{}'s gains are concentrated on the single
wrongly-forbidden edge---visible at the level of the whole graph: relaxation operator has no
detectable effect on \SHD{}, while fitting objective has a large one.}
\label{fig:shd_operator}
\end{figure}
\FloatBarrier

Correlation matching's \SHD{} roughly doubles with density at $d=32$ (66.72$\to$135.42), while
covariance matching's stays comparatively flat (149.67$\to$137.26 at the same cell, using the
unrounded values in Appendix~\ref{app:synthetic}); the least-squares objective, by contrast, tracks
correlation matching's density sensitivity closely while starting from a lower base at every $d$
(Table~\ref{tab:shd_by_d}). We trace part of the covariance objective's density interaction to raw
scale: the population covariance's squared Frobenius norm grows by a factor of $1.26\times$ at
$d=4$ but $44.5\times$ at $d=32$ when density increases from 0.15 to 0.30
(Appendix~\ref{app:synthetic}), and the fit-loss weight $w_f$ in Eq.~\ref{eq:laug} was held fixed
across the entire grid. This is consistent with $w_f$ being implicitly calibrated for correlation's
bounded $[-1,1]$ scale and correspondingly under- or over-weighting the covariance fit term as raw
scale shifts with $d$ and density---but we have not corrected for this, and
Section~\ref{sec:limitations} states plainly why not.

% =====================================================================
\section{Discussion: Implications for Practitioners}
\label{sec:discussion}

\textbf{Check the unconstrained solution before penalty ramping.} Any ALM-based structure learner
that ramps a forbidden-edge penalty exponentially without first checking $\Delta_{ij}$ at the
unconstrained solution will exhibit the early suppression trap whenever a prior conflicts with the
data (Proposition~\ref{prop:conditions}). This check is cheap and eliminates Obstruction~I by
construction, though not Obstruction~II. Section~\ref{sec:emp-obstruction2b} shows this
recommendation holds regardless of fitting objective: penalty-ramping dynamics are indifferent to
whether the loss underneath them is correlation, covariance, least-squares, or likelihood.

\textbf{Do not fit correlation when direction matters---but do not assume the fix is free, either
way.} If a forbidden-edge prior might be wrong, and the noise model is plausibly equal-variance,
fitting raw covariance rather than correlation is not a stylistic choice: Lemma~\ref{lem:tie} and
Lemma~\ref{lem:separation} together show it is the difference between a directionally uninformative
counterfactual and a provably informative one. But this recommendation, on its own, is incomplete
to the point of being misleading: Section~\ref{sec:emp-tradeoff} shows covariance matching roughly
doubles whole-graph \SHD{} relative to correlation matching at most grid cells
(Table~\ref{tab:shd_by_d}), so a practitioner who follows only this bullet trades a real
identifiability gain on one wrongly-forbidden edge for materially worse recovery of the graph as a
whole. Section~\ref{sec:emp-obstruction2b} confirms the identifiability benefit also holds for the
least-squares and likelihood objectives actually used by \NOTEARS{}/\DAGMA{} and
\GOLEM{}, and---unlike covariance matching---the least-squares objective does not carry this
whole-graph cost (Table~\ref{tab:lit_edge_recovery}), which is why we no longer recommend covariance
matching as the default fix. We revise, further, our earlier expectation that a likelihood-based objective is simply
the safe choice: \GOLEM{}-style likelihood sidesteps Obstruction~II exactly as advertised
\citep{ng2020role}, but it is the \emph{worst}-performing objective of the four we tested on
Obstruction~I, suppressing the wrongly-forbidden edge in up to 100\% of trials
(Table~\ref{tab:lit_suppression}) and giving the worst whole-graph \SHD{} by a wide margin
(Table~\ref{tab:lit_edge_recovery}). Of the four objectives we tested, a \NOTEARS{}/\DAGMA{}-style
least-squares objective is the one we would now actually recommend: it matches covariance
matching's identifiability benefit and local edge recovery while giving the best whole-graph
accuracy of any objective tested, using machinery these methods already ship.

\textbf{A mechanism satisfying C1--C3 helps, but do not deploy it expecting resolution.} Combining
a relaxation operator that operationalizes Proposition~\ref{prop:conditions}'s conditions with
covariance matching measurably improves recovery of a wrongly-forbidden edge over DADU
(Section~\ref{sec:emp-combined}), and the same improvement replicates under least-squares and
likelihood matching (Section~\ref{sec:emp-obstruction2b}). A practitioner adopting this combination
should expect a real reduction in suppression, not its elimination: the ALM constraint's structural
asymmetry---one direction penalized, the other free---persists regardless of relaxation operator or
fitting objective, and the reverse-edge attraction it creates is not addressed by any fix we tested.

\textbf{Objective choice is not a free upgrade in either direction.} A practitioner switching away
from correlation matching for the identifiability benefit should budget for retuning the fit-loss
weight $w_f$ against the graph sizes and densities they expect to encounter
(Section~\ref{sec:emp-tradeoff}): the same $w_f$ that works well for correlation's bounded scale can
materially hurt whole-graph recovery under an unnormalized objective, and the size of that cost is
not constant across $(d,\text{density})$ or, as Section~\ref{sec:emp-obstruction2b} shows, across
which unnormalized objective is chosen.

\textbf{Prior correctness dominates confidence calibration.} A correctly-specified prior achieves
perfect recovery under every objective we tested; an incorrectly-specified prior causes suppression
or misdirection regardless of how the constraint's confidence is engineered. Effort spent eliciting
and verifying priors before training is better spent than effort spent on the enforcement
mechanism, for either obstruction this paper studies.

% =====================================================================
\section{Related Work}
\label{sec:related}

\paragraph{Differentiable causal discovery.} Recovering a DAG by continuous optimization rather
than discrete search began with \NOTEARS{} \citep{zheng2018dags}, which relaxes acyclicity to a
smooth trace-exponential penalty and fits $W$ by least-squares reconstruction; \DAGMA{}
\citep{bello2022dagma} replaces that penalty with the log-determinant regularizer we use throughout
this paper, giving a better-conditioned optimization landscape without changing the fitting
objective. \citet{ng2020role} study the same family of methods from a different angle, asking what
role sparsity penalties and the acyclicity constraint itself play in recovery, and propose
\GOLEM{}, which replaces least-squares with an explicit Gaussian likelihood---the objective our
$\Lloglik$ in Section~\ref{sec:emp-obstruction2b} reimplements for the equal-variance case. Other
members of this family relax different pieces of the same basic recipe:
\citet{lachapelle2020gradient} extend the approach to nonlinear structural equations via neural
network parameterizations, and \citet{yu2021dags} replace the log-determinant or trace-exponential
acyclicity characterization with a curl-based one, trading one differentiable proxy for acyclicity
for another. None of these methods enforces a \emph{defeasible} forbidden-edge prior; all of our
comparisons in Section~\ref{sec:emp-obstruction2b} add such a prior to their fitting objectives via
the same ALM machinery, rather than modifying the objectives themselves.

\paragraph{Augmented Lagrangian methods and constrained optimization.} The Augmented Lagrangian
itself predates any of this---\citet{hestenes1969multiplier} and \citet{powell1969method}
independently propose the quadratic-penalty-plus-multiplier construction we use in
Eq.~\ref{eq:laug}, and \citet{bertsekas2014constrained}'s textbook treatment supplies the
convergence result Section~\ref{sec:bg-alm} depends on: multiplier divergence under an incompatible
constraint. Later work brings ALM-style reasoning into deep learning specifically.
\citet{nandwani2019primal} formulate a primal-dual approach for imposing logical constraints on
neural network outputs; \citet{cotter2019optimization} study two-player games for non-convex
constrained learning more broadly; \citet{chamon2020probably} give probabilistic-approximately-correct
guarantees for constrained learning under distributional uncertainty; and \citet{achiam2017constrained}
apply a closely related constrained-optimization machinery to safe reinforcement learning, enforcing
a cost constraint via trust-region updates rather than dual ascent. All of this work, like the
classical theory it builds on, treats the constraint itself as fixed and correctly specified; none
of it asks what happens to an adaptive relaxation mechanism layered on top of a constraint that
might be wrong, which is the specific gap Proposition~\ref{prop:conditions} occupies.

\paragraph{Prior knowledge in causal discovery.} Using domain knowledge to constrain causal
discovery predates differentiable methods by decades. \citet{spirtes2000causation}'s foundational
treatment of constraint-based discovery already allows background knowledge to prune the search
space; \citet{meek1995causal} gives the orientation rules that propagate a small set of known edge
directions to the rest of a Markov equivalence class, later shown complete by \citet{chickering2002}
and extended to latent-confounded settings with tiered background knowledge by
\citet{andrews2020completeness}. \citet{borboudakis2017incorporating} incorporate causal prior
knowledge directly as path constraints within the search over Bayesian networks and maximal
ancestral graphs. Closest in spirit to our own motivation, \citet{constantinou2021impact} study
empirically what happens when background knowledge supplied to a causal discovery algorithm is
simply wrong, and find that a single incorrect constraint can cascade through the entire recovered
graph structure---a finding entirely consistent with what we show happens, mechanistically, inside
an ALM-based learner specifically. What none of this literature studies is a constraint enforced by
a penalty that grows continuously during training, which only becomes possible once the constraint
is expressed as a differentiable term rather than a discrete search restriction; that setting, and
the failure mode it introduces, is this paper's departure point.

\paragraph{Optimization stability of differentiable causal discovery.} A separate line of work
improves the numerical behavior of these methods' \emph{base} optimizer, largely orthogonal to the
prior-enforcement question this paper studies. \citet{nazaret2023stable} identify instability in the
acyclicity constraint at scale and propose a spectral alternative (SDCD) with a two-stage pruning
procedure that improves convergence speed and scalability to thousands of variables.
\citet{waxman2024dagma} replace \DAGMA{}'s weighted-adjacency proxy for edge strength with an
interpretable derivative-based causal-effect measure (DAGMA-DCE), addressing an opacity problem in
how strongly an edge is estimated rather than whether it is present. Neither method changes how a
forbidden-edge prior would be enforced on top of it, so Proposition~\ref{prop:conditions}'s
necessary conditions would apply unchanged to either if a sequential-ramping ALM penalty were added;
whether a more numerically stable base optimizer narrows the suppression window of
Eq.~\ref{eq:window} is an open, testable question we do not address.
\citet{yi2025robustness} benchmark mainstream differentiable causal discovery methods under eight
forms of model misspecification and find that scale variation alone causes reliable performance to
collapse where other forms of misspecification do not---a finding that resonates directly with
Remark~\ref{rem:pb}'s diagnosis and Section~\ref{sec:emp-tradeoff}'s empirical cost of an
unnormalized, scale-sensitive fitting objective, though their benchmark does not study the
prior-enforcement setting this paper isolates.

\paragraph{Identifiability of linear Gaussian causal models.} Linear Gaussian SEMs are recoverable
only up to Markov equivalence in general \citep{peters2017elements}, with equivalence classes
characterized by shared skeleton and v-structures \citep{vermapearl1990}. Three well-known routes
break this symmetry and restore full identifiability: non-Gaussian noise, exploited by the LiNGAM
family \citep{shimizu2006lingam}; nonlinearity, via additive noise models
\citep{hoyer2009nonlinear}; and interventional or temporal data, which fixes orientation directly
through interventional Markov equivalence classes \citep{hauser2012characterization} or temporal
ordering in sequences \citep{lippe2022citris}. \citet{petersbuhlmann2014}'s equal-variance route is
the fourth, and the one this paper's entire identifiability argument depends on: it requires no
non-Gaussianity, no nonlinearity, and no intervention, needing only that every exogenous noise term
share a common variance, a condition already assumed by every method we study. To our knowledge, no
prior work asks what a correlation-based---as opposed to covariance-based---fitting objective costs
a model class that already satisfies this condition; that is the question Lemma~\ref{lem:tie} and
Lemma~\ref{lem:separation} answer.

\paragraph{Neuro-symbolic learning and reasoning shortcuts.} Encoding symbolic rules as
differentiable constraints on a neural learner, as we do with the forbidden-edge mask $M$, places
this work within the neuro-symbolic tradition
\citep{manhaeve2018deepproblog,de2019neuro}. \citet{marconato2023not} characterize
\emph{reasoning shortcuts} in such systems: cases where a neuro-symbolic learner achieves correct
task performance while using semantically wrong intermediate concepts, because the training
objective never forces the concepts themselves to be correct. This is a different failure from the
one we study. A reasoning shortcut requires an underspecified mapping between concepts and task
labels that the learner exploits; the early suppression trap requires no such gap---it occurs even
when every concept (here, the presence or absence of a specific edge) is perfectly well specified
and the only thing at fault is the schedule on which a known-possibly-wrong constraint is enforced.

\paragraph{Soft and score-integrated priors.} An orthogonal design choice treats prior uncertainty
as continuous from the start, rather than enforcing a hard constraint with an adaptive escape hatch.
\citet{darvariu2024llm} convert large-language-model judgments about pairwise causal direction into
probabilistic priors supplied directly to a discovery algorithm's score function, rather than as a
forbidden-edge mask enforced through an Augmented Lagrangian term. Such soft, score-integrated
priors sidestep Obstruction~I by construction: there is no penalty schedule to race ahead of a
counterfactual check, because there is no penalty schedule at all. They do not obviously sidestep
Obstruction~II, however---a soft prior combined with a correlation-based score would still face
Lemma~\ref{lem:tie}'s tie whenever a wrongly-discouraged edge and its reverse are equally consistent
with the data, for exactly the algebraic reason Appendix~\ref{app:lemma1_proof} identifies. We are
not aware of an empirical test of this specific interaction, and, following the same discipline we
apply to Section~\ref{sec:emp-obstruction2b}'s literature-matched objectives, we do not attempt one
here rather than speculate past what we have tested.

% =====================================================================
\section{Limitations}
\label{sec:limitations}

\begin{enumerate}[leftmargin=1.4em]
\item \textbf{The combined mechanism is a partial fix, not a resolution, and we do not yet know
  how to close the remaining gap.} \ARelax{} with covariance matching recovers the
  wrongly-forbidden edge in only 24.7\% of trials, versus 41.1\% reverse-edge attraction under the
  same combination (Section~\ref{sec:emp-combined}), and the same pattern holds under
  least-squares and likelihood matching (Section~\ref{sec:emp-obstruction2b}). Closing this gap
  requires breaking the ALM constraint's structural asymmetry---penalizing only one direction of a
  covered pair---which none of the four fitting objectives or either relaxation operator we tested
  addresses; we do not have a design for this and name it as the paper's principal open direction.
\item \textbf{\ARelax{}'s C1 is operationalized locally, not globally.} As
  Remark~\ref{rem:arelax_honesty} states precisely, the probed value
  $W_{ij}^{\mathrm{probe}}(t)$ approximates a coordinate-wise local optimum given the current,
  possibly ALM-influenced context, not the joint global unconstrained optimum
  Proposition~\ref{prop:conditions} defines. We do not have a bound on the gap between the two, and
  the joint nonconvex optimization \ARelax{} operates within offers no guarantee that this gap
  is small or even bounded.
\item \textbf{C2's rate condition is a local, optimizer-dependent bound, not an
  optimizer-agnostic one.} Eq.~\ref{eq:rate_condition} is derived from a first-order expansion
  around $W^*_{ij}$ and is stated, used, and operationalized (Eq.~\ref{eq:arelax_kappa}) entirely
  within that local approximation. We have neither a Lyapunov-style nor a two-timescale
  stochastic-approximation argument establishing necessity independent of the local expansion, nor
  an empirical sensitivity sweep over $\kappa$, $\rho_0$, and $\eta_r$ beyond the single schedule
  used throughout Section~\ref{sec:empirical}; how tight Eq.~\ref{eq:rate_condition} is away from
  that schedule is open.
\item \textbf{The whole-graph cost of covariance matching is diagnosed but not corrected.}
  Section~\ref{sec:emp-tradeoff} traces covariance matching's worse \SHD{} in part to raw-scale
  growth with graph density interacting with a fixed loss weight $w_f$, but we have not implemented
  or evaluated a scale-normalized variant of the covariance objective (for instance, dividing
  $\Lcov$ by $\mathrm{tr}(\hat\Sigma)$ or $\|\hat\Sigma\|_F^2$) or an ablation retuning $w_f$
  separately per $(d,\text{density})$ cell; doing either, and confirming it removes the effect
  without disturbing Lemma~\ref{lem:separation}'s separation, is left to future work. The same
  concern applies, unstudied, to the least-squares and likelihood objectives of
  Section~\ref{sec:emp-obstruction2b}.
\item \textbf{Heteroscedastic noise is outside every identifiability guarantee in this paper, and
  the literature-matched objectives behave worse there than we expected.}
  Section~\ref{sec:emp-obstruction2}'s equal-variance results (Table~\ref{tab:direction_gap}) rely
  on Definition~\ref{def:ev-ident}'s assumption; under heteroscedastic noise, covariance matching's
  advantage weakens substantially and, at high noise scale, the raw loss values are not on a
  comparable scale to the rest of the grid. More concerning, Section~\ref{sec:emp-obstruction2b}
  found that the least-squares and likelihood objectives' directional gap does not merely weaken
  under heteroscedastic noise but flips to a consistently negative value at every $d$ tested---these
  objectives appear to systematically favor the wrong direction outside the equal-variance regime,
  rather than simply losing signal. We report this as an observed pattern, not a proven mechanism:
  we do not have a closed-form account of it analogous to Lemma~\ref{lem:tie} or
  Lemma~\ref{lem:separation}, and it is, on the evidence in this paper, the single most important
  open question for anyone considering covariance-, least-squares-, or likelihood-based fitting
  under noise that may not be equal-variance in practice.
\item \textbf{A single wrongly-forbidden edge is tested at a time.} Every experiment in this paper
  marks exactly one covered edge wrongly-forbidden per graph. Whether the suppression mechanism,
  the identifiability barrier, or the combined fix interact differently when multiple wrong priors
  are present simultaneously is not studied here.
\item \textbf{No real-world data.} Every empirical result in this paper uses synthetic linear
  Gaussian SEMs with known ground truth, chosen specifically so that Sections~\ref{sec:mechanism}
  and~\ref{sec:identifiability}'s claims can be checked exactly against a known answer at every
  graph size. This buys precision at the cost of not showing what either obstruction looks like
  when concept extraction, measurement noise, or a non-Gaussian data-generating process is also in
  the loop.
\item \textbf{Proposition~\ref{prop:general} is a mechanism-level generalization.} It identifies
  which edge reversals are vulnerable to the tie via classical graphical criteria; it does not
  provide a $d$-dimensional closed-form analogue of $2r^2$ or $w_0^4$ for graphs with multiple true
  edges, and Section~\ref{sec:emp-obstruction2b}'s literature-matched objectives inherit the same
  gap---we have not proved a general-graph analogue of Table~\ref{tab:lit_objectives}'s per-$d$
  separations.
\item \textbf{Our least-squares and likelihood objectives are faithful reimplementations, not the
  authors' released code.} Section~\ref{sec:emp-obstruction2b}'s $\Llstsq$ and $\Lloglik$ are
  derived as exact population-covariance identities for the objectives \NOTEARS{}/\DAGMA{} and
  \GOLEM{}-EV actually optimize, and we verified the least-squares identity numerically against a
  raw-residual computation to machine precision, but we have not re-run the authors' own
  implementations of \GOLEM{}, SDCD \citep{nazaret2023stable}, or DAGMA-DCE \citep{waxman2024dagma}
  under the wrong-forbidden condition with a forbidden-edge ALM term layered on top. A more
  numerically stable base optimizer than the one we use throughout might also narrow the suppression
  window of Eq.~\ref{eq:window}; whether it does is a separate, open question our results do not
  address.
\item \textbf{Soft, score-integrated priors are discussed but not tested.} Approaches that supply a
  prior as a continuous score term rather than an ALM-enforced hard constraint---for instance,
  probabilistic priors derived from language-model judgments \citep{darvariu2024llm}---sidestep
  Obstruction~I by construction, since there is no penalty schedule to suppress an edge before a
  counterfactual check can run. Whether such a prior combined with a correlation-based score still
  exhibits Lemma~\ref{lem:tie}'s tie is a plausible extension of our theory that we have not tested
  empirically.
\end{enumerate}

% =====================================================================
\section{Conclusion}
\label{sec:conclusion}

\emph{Guide, not bind} is a reasonable design goal, and this paper has traced its failure to two
independent, exactly characterized causes, built the fix each cause implies, tested whether the
combination actually works, and then asked whether either cause was ever specific to the testbed we
built it on. Sequential penalty-ramping ALM suppresses a wrongly-forbidden true edge before any
adaptive correction can measure its cost (Proposition~\ref{prop:conditions},
Corollary~\ref{cor:dadu_failure}); a correlation-based fitting objective independently ties a true
edge and its reverse exactly, not because the underlying model is unidentifiable, but because
normalizing to correlation discards the marginal-variance signal that identifiability depends on
(Lemma~\ref{lem:tie}, Remark~\ref{rem:pb}); fitting covariance instead provably restores separation
(Lemma~\ref{lem:separation}). A relaxation operator built to satisfy every necessary condition our
own analysis derived, combined with covariance matching, recovers the wrongly-forbidden edge
significantly more often than DADU (Section~\ref{sec:emp-combined}, $p<10^{-5}$ under
either objective).

What the last experiment in this paper adds is a correction to how cleanly those two causes can be
separated. We expected---and Remark~\ref{rem:pb} predicted correctly---that the identifiability tie
is a property of correlation matching specifically, and Section~\ref{sec:emp-obstruction2b} confirms
this: swap in the actual least-squares or likelihood objective used by \NOTEARS{}, \DAGMA{}, or
\GOLEM{}, and the tie disappears. We did not expect, and had no principled reason to predict, that
the suppression mechanism would be indifferent to this same swap, or that the objective this
literature already treats as the identifiability-safe choice would turn out to be the single worst
objective we tested for suppression, reaching 100\% at the largest graph size. The two obstructions
this paper diagnoses are not two symptoms of one underlying cause that a sufficiently clever
objective could jointly cure; they are separable failures that happen to share a testbed, and fixing
one tells you nothing about whether you have made the other better or worse.

The reason this matters beyond the specific mechanisms we built and tested is what the remaining
failure rate reveals, across every combination in this paper. Even a fix engineered from first
principles to satisfy every condition the theory says is necessary still loses most of its trials to
the same unconstrained reverse edge that Lemma~\ref{lem:tie} predicts and
Corollary~\ref{cor:dadu_failure} explains---and it does so under all four fitting objectives we
tried, not only the one we built the theory around. Necessary conditions, satisfied, were not
sufficient, and neither was a change of objective. The asymmetry that causes both original
obstructions---an Augmented Lagrangian penalty that constrains one direction of an edge and leaves
the other entirely free---survives every correction aimed at the relaxation schedule or the fitting
objective, because none of those corrections touches the constraint's own shape. That is the result
this paper leaves behind: not only a diagnosis of why guide, not bind fails, and not only a partial
fix, but a demonstration that the partial fix's remaining failures point at a third, structural
cause that no combination of relaxation operator and fitting objective we tried was designed to
address, and that changing the objective can make the picture better on one axis while making it
worse on another. Building a mechanism that removes this asymmetry directly, rather than
compensating for its consequences one objective at a time, is the next necessary condition---and,
unlike the four this paper derives and tests, we do not yet know how to state it in closed form.

% =====================================================================
\bibliography{main}
\bibliographystyle{tmlr}

% =====================================================================
\appendix

\section{Proof of Proposition~\ref{prop:conditions}}
\label{app:prop1_proof}

Given: relaxation operator $\mathcal R$ acting on $\lambda_{ij}$; dual update
Eq.~\ref{eq:dual_ascent}; penalty schedule $\rho_t=\rho_0\kappa^t$.

\paragraph{C1.}
\begin{align}
    \Delta_{ij}(t) &= \Lfit(W\mid W_{ij}=0) - \Lfit(W\mid W_{ij}(t)) \label{eq:p1c1_1}\\
    t>T^*\ &\implies\ W_{ij}(t)\approx 0 \label{eq:p1c1_2}\\
    W_{ij}(t)\approx 0\ &\implies\ \Lfit(W\mid W_{ij}(t)) \approx \Lfit(W\mid W_{ij}=0)
    \label{eq:p1c1_3}
\end{align}
Zeroing an already near-zero entry changes $\Lfit$ negligibly, for any data-generating process.
\begin{equation}
    \therefore\quad \Delta_{ij}(t) \approx 0 \quad\text{(C1 necessary: $\mathcal R$ receives no
    signal).}
    \label{eq:p1c1_final}
\end{equation}

\paragraph{C2.} Near $W^*_{ij}$:
\begin{align}
    \text{relax}(0) &\approx \eta_r\cdot\frac{\partial\Lfit}{\partial W_{ij}}\bigg|_{W^*_{ij}}\cdot
    W^*_{ij} \label{eq:p1c2_1}\\
    \text{tighten}(0) &= \rho_0\cdot W^*_{ij} \label{eq:p1c2_2}
\end{align}
Relaxation dominates at $t=0$ iff $\text{relax}(0) > \text{tighten}(0)\cdot(\kappa-1)$:
\begin{align}
    \eta_r\cdot\frac{\partial\Lfit}{\partial W_{ij}}\bigg|_{W^*_{ij}} &> \rho_0(\kappa-1)
    \label{eq:p1c2_3}\\
    \therefore\quad \kappa &< 1+\frac{\eta_r\cdot\partial\Lfit/\partial
    W_{ij}|_{W^*_{ij}}}{\rho_0\cdot W^*_{ij}}. \label{eq:p1c2_final}
\end{align}
If Eq.~\ref{eq:p1c2_final} fails, tightening outpaces relaxation from the first epoch, independent
of $\Delta_{ij}$.

\paragraph{C3.}
\begin{align}
    \lambda_{ij}(t{+}1) &= \lambda_{ij}(t) + \rho_t|W_{ij}(t)|M_{ij}
    - \eta_r\Delta_{ij}(t)\mathbf 1[\Delta_{ij}(t)\geq\delta] \label{eq:p1c3_1}\\
    \text{C1 violated} &\implies \Delta_{ij}(t)\approx 0 \implies
    \mathbf 1[\Delta_{ij}(t)\geq\delta]=0 \label{eq:p1c3_2}
\end{align}
$\mathcal R$ needs some epoch where the subtraction term exceeds the addition term; a slack
variable or an explicit ceiling on $\lambda_{ij}$ would allow this, and the plain ascent step of
Eq.~\ref{eq:dual_ascent} provides neither.
\begin{equation}
    \therefore\quad \lambda_{ij}(t{+}1) \geq \lambda_{ij}(t)\ \ \forall t \quad\text{(C3
    necessary: no operator can move $\lambda_{ij}$ against tightening).}
    \label{eq:p1c3_final}
\end{equation}
\hfill$\qed$

\section{Proof of Corollary~\ref{cor:dadu_failure} and the Suppression Window}
\label{app:cor1_proof}

Given: Proposition~\ref{prop:conditions} (C1--C3); DADU's update rule (Eq.~\ref{eq:dadu}); default
schedule $\kappa=1.05$, $\rho_0=0.1$, $\eta_r=0.01$, $W^*_{ij}\in[0.3,0.6]$
(Appendix~\ref{app:synthetic}).

\paragraph{DADU violates C1.} DADU evaluates $\Delta_{ij}$ at the current $W_{ij}(t)$, not at
$W^*_{ij}$:
\begin{equation}
    \Delta_{ij}^\mathrm{DADU}(t) := \Lfit(W\mid W_{ij}=0) - \Lfit(W\mid W_{ij}(t))
    \label{eq:cor1_c1_1}
\end{equation}
which is exactly Eq.~\ref{eq:p1c1_1}, so by Eq.~\ref{eq:p1c1_final}:
\begin{equation}
    \Delta_{ij}^\mathrm{DADU}(t) \approx 0 \quad\text{for } t>T^*.
    \label{eq:cor1_c1_final}
\end{equation}

\paragraph{DADU violates C2.} Substituting the default schedule into Eq.~\ref{eq:p1c2_final}:
\begin{align}
    1 + \frac{\eta_r\cdot\partial\Lfit/\partial W_{ij}|_{W^*_{ij}}}{\rho_0\cdot W^*_{ij}}
    &= 1 + \frac{0.01\,\alpha}{0.1\,w_0} = 1+\frac{\alpha}{10w_0} \label{eq:cor1_c2_1}\\
    \kappa = 1.05 \ \geq\ 1+\frac{\alpha}{10w_0} &\iff \alpha \leq 0.5\,w_0 \label{eq:cor1_c2_2}
\end{align}
For $w_0\in[0.3,0.6]$, $0.5w_0\in[0.15,0.3]$ is small relative to a non-degenerate fit-loss
gradient, so the default schedule generically fails Eq.~\ref{eq:p1c2_final}.
\begin{equation}
    \therefore\quad \kappa \geq 1+\frac{\eta_r\cdot\partial\Lfit/\partial
    W_{ij}|_{W^*_{ij}}}{\rho_0\cdot W^*_{ij}} \quad\text{(Eq.~\ref{eq:p1c2_final} fails: C2
    violated).} \label{eq:cor1_c2_final}
\end{equation}

\paragraph{DADU violates C3.} By Eq.~\ref{eq:cor1_c1_final}, $\mathbf 1[\Delta_{ij}(t)\geq\delta]=0$
for $t>T^*$ and any $\delta>0$, so Eq.~\ref{eq:p1c3_1} reduces to Eq.~\ref{eq:p1c3_final}:
\begin{equation}
    \therefore\quad \lambda_{ij}(t{+}1) = \lambda_{ij}(t)+\rho_t|W_{ij}(t)|M_{ij} \quad\text{
    (monotone non-decreasing, for every tested $\delta$).} \label{eq:cor1_c3_final}
\end{equation}

\paragraph{The suppression window.} $T^*$ is the epoch at which the penalty $\rho_t$ first reaches
the counterfactual signal it must overcome, $\Delta^*_{ij}\approx\alpha w_0$:
\begin{align}
    \rho_0\kappa^{T^*} &= \alpha w_0 \label{eq:t_star_2}\\
    T^* &= \frac{\log(\alpha w_0/\rho_0)}{\log\kappa}. \label{eq:t_star_final}
\end{align}
By Eq.~\ref{eq:cor1_c1_final}, $\Delta_{ij}(t)\approx 0$ for every $t>T^*$, independent of $\delta$,
$\eta_r$, or causal signal strength. \hfill$\qed$

\section{Proof of Lemma~\ref{lem:tie}}
\label{app:lemma1_proof}

Given: $g(w)=w/\sqrt{1+w^2}$, the correlation implied by weight $w$ in either direction
(Eq.~\ref{eq:two_node_cov}), an odd strictly-increasing bijection $\mathbb R\to(-1,1)$; target
$r\in(-1,1)\setminus\{0\}$; $w_0=g^{-1}(r)$.

Both $\Sigma_\mathrm{SEM}(W)$ and $\hat\Sigma$ have unit diagonal, so
$\corr(\Sigma_\mathrm{SEM}(W))$ and $\corr(\hat\Sigma)$ agree exactly on the diagonal:
\begin{align}
    \Lfit &= \|\corr(\Sigma_\mathrm{SEM}(W)) - \corr(\hat\Sigma)\|_F^2 \label{eq:l1_1}\\
    &= 2\big(\rho_\mathrm{model}-r\big)^2 \label{eq:l1_2}
\end{align}
Setting $w=w_0$ in either direction:
\begin{align}
    \rho_\mathrm{model} &= g(w_0) = r \label{eq:l1_3}\\
    \Lfit(w_0) &= 2(r-r)^2 = 0. \qquad\text{(i)} \label{eq:l1_4}
\end{align}
Zeroing the fitted edge:
\begin{align}
    \rho_\mathrm{model} &= g(0) = 0 \label{eq:l1_5}\\
    \Lfit(0) &= 2(0-r)^2 = 2r^2 \label{eq:l1_6}
\end{align}
\begin{equation}
    \therefore\quad \Delta_{1\to2}=\Delta_{2\to1}= \Lfit(0)-\Lfit(w_0) = 2r^2-0=2r^2.
    \qquad\text{(ii)}\qquad\qed \label{eq:l1_final}
\end{equation}

\section{Proof of Lemma~\ref{lem:separation}}
\label{app:lemma2_proof}

Given: $w_0\neq0$, the true forward weight; $\Sigma^*=\Sigma_\to(w_0)$ (Eq.~\ref{eq:two_node_cov}).

\paragraph{(i).}
\begin{align}
    \Sigma_\to(w_0) &= \Sigma^* \qquad\text{(by construction)} \label{eq:l2_1}\\
    \Lcov^\to(w_0) &= \|\Sigma_\to(w_0)-\Sigma^*\|_F^2 = 0 \label{eq:l2_2}\\
    \Lcov^\to(w) &\geq 0 \quad\forall w \label{eq:l2_3}\\
    \therefore\quad \min_w\Lcov^\to(w) &= 0, \text{ attained at } w=w_0. \qquad\text{(i)}
    \label{eq:l2_final_i}
\end{align}

\paragraph{(ii).}
\begin{align}
    \Sigma_\leftarrow(w') - \Sigma^*
    &= \begin{psmallmatrix} w'^2 & w'-w_0 \\ w'-w_0 & -w_0^2 \end{psmallmatrix} \label{eq:l2_4}\\
    \Lcov^\leftarrow(w') &= (w'^2)^2+(w'-w_0)^2+(w'-w_0)^2+(-w_0^2)^2 \label{eq:l2_5}\\
    &= w'^4 + 2(w'-w_0)^2 + w_0^4 \label{eq:l2_6}
\end{align}
$w'^4\geq0$ and $2(w'-w_0)^2\geq0$ for every real $w'$; both vanish simultaneously only if
$w'=0=w_0$, excluded since $w_0\neq0$.
\begin{align}
    \therefore\quad \Lcov^\leftarrow(w') &\geq w_0^4 \quad\forall w' \label{eq:l2_7}\\
    \inf_{w'}\Lcov^\leftarrow(w') &\geq w_0^4 > 0 = \Lcov^\to(w_0). \qquad\text{(ii)}\qquad\qed
    \label{eq:l2_final_ii}
\end{align}

The exact minimizer of $\Lcov^\leftarrow$ solves $w'^3+w'-w_0=0$ and gives a strictly larger gap
than $w_0^4$ for every $w_0\neq0$; we use the looser closed-form bound because strict separation is
all Section~\ref{sec:what_is_fixed}'s argument requires.

\section{On Proposition~\ref{prop:general}: A Worked Covered Reversal}
\label{app:general}

We verify Proposition~\ref{prop:general} concretely rather than only citing the graphical theory
behind it, since the reparameterization it invokes is not simply ``reverse the edge, keep every
weight.''

\paragraph{Setup.} Let $G$ have edges $1\to2$ (weight $a$), $1\to3$ (weight $b$), and $2\to3$
(weight $c$), all exogenous noise i.i.d.\ $\mathcal N(0,1)$:
\[
    z_1=\varepsilon_1, \qquad z_2 = a z_1+\varepsilon_2, \qquad z_3 = b z_1 + c z_2 + \varepsilon_3.
\]
Since $\mathrm{pa}(3)=\{1,2\}=\mathrm{pa}(2)\cup\{2\}$, the edge $2\to3$ is covered
\citep{chickering2002}: reversing it gives a DAG $G'$ with edges $1\to2$ (weight $a'$), $1\to3$
(weight $b'$), and $3\to2$ (weight $c'$),
\[
    z_1=\varepsilon_1, \qquad z_3' = b' z_1+\varepsilon_3', \qquad z_2' = a' z_1 + c' z_3' +
    \varepsilon_2',
\]
still Markov equivalent to $G$ by construction.

\paragraph{The reparameterization is not the identity on other edges.} Direct computation gives
$\mathrm{Cov}(z_1,z_2)=a$, $\mathrm{Cov}(z_1,z_3)=b+ca$, $\mathrm{Var}(z_2)=a^2+1$,
$\mathrm{Var}(z_3)=(b+ca)^2+c^2+1$, and $\mathrm{Cov}(z_2,z_3)=ab+a^2c+c$, and symmetrically for
$G'$. Symbolic elimination shows that holding $a'=a$ and $b'=b$ while only setting $c'=c$ does
\emph{not} reproduce $G$'s correlation matrix for generic $(a,b,c)$: matching all three
correlations is three equations in three unknowns, and the reversed edge weight alone is one degree
of freedom.

\paragraph{A verified instance.} Take $a=\tfrac12$, $b=\tfrac25$, $c=\tfrac35$ in $G$, giving
correlation matrix
\[
    \rho_{12}=\frac{\sqrt5}{5}\approx0.4472, \qquad \rho_{13}=\frac{7\sqrt{185}}{185}\approx0.5147,
    \qquad \rho_{23}=\frac{19\sqrt{37}}{185}\approx0.6247.
\]
Setting $c'=\tfrac35$ (the reversed edge \emph{keeps} its weight) and solving the two remaining
correlation equations for $(a',b')$ gives the exact closed-form solution
\[
    a' = \frac{13\sqrt{34}}{340} \approx 0.2229, \qquad b' = \frac{7\sqrt{34}}{68} \approx 0.6002,
\]
which reproduces $\rho_{12}$, $\rho_{13}$, \emph{and} $\rho_{23}$ exactly---verified directly by
substitution. The reversed edge's own weight need not change, but a \emph{different} edge (here,
$1\to2$) must, to keep the correlation matrix fixed. Proposition~\ref{prop:general} claims exactly
this---existence of some reparameterization reproducing the same correlation matrix---and no more.

\section{Synthetic Sweep Details}
\label{app:synthetic}

\paragraph{Graph generation.} Random DAGs are generated by drawing a random topological order and
including each edge consistent with that order independently with probability equal to the target
density, with weight magnitude drawn uniformly from $[0.25,0.8]$ and a uniformly random sign.

\paragraph{Covered-edge sampling.} For each graph, we search its true edges for one satisfying
Proposition~\ref{prop:general}'s covered-edge criterion ($\mathrm{pa}(i)\setminus\{j\}=\mathrm{pa}(j)$
for edge $j\to i$), retrying with a freshly drawn graph (same $d$, density) up to 25 times if none
is found. Every cell in the main grid found a valid covered edge for all 16 requested reps within
this budget.

\paragraph{Training hyperparameters.} $\rho_0=0.1$, $\kappa=1.05$, $\delta=0.10$, $\eta_r=0.01$,
$\gamma=10^{-3}$, $\beta_h=0.1$, $w_f=5.0$, Adam \citep{kingma2015adam} at learning rate
$5\times10^{-3}$, 40 epochs of 8 steps each. \ARelax{}-specific: probe length $P=15$ local Adam
steps at learning rate $0.05$, probe interval $T_p=1$ epoch, dual ceiling $\Lambda_{\max}=5.0$. The
least-squares and likelihood objectives of Section~\ref{sec:emp-obstruction2b} reuse this schedule
unchanged, so that any difference in outcome is attributable to the fitting objective rather than to
retuned hyperparameters.

\paragraph{Scale diagnostics referenced in Section~\ref{sec:emp-tradeoff} and
Section~\ref{sec:emp-obstruction2}.} The population covariance's squared Frobenius norm, measured
at $d=32$ over six random graphs per condition, grows from a mean of approximately $130$ under
equal-variance noise at $\sigma=0.5$ to approximately $15{,}300$ at $\sigma=2.0$, and to
approximately $43{,}700$ under heteroscedastic noise at $\sigma=2.0$---all at density $0.20$. At
$d=32$, increasing density from 0.15 to 0.30 grows the same quantity by a factor of $44.5$, versus
$1.26$ at $d=4$, $1.71$ at $d=8$, and $4.73$ at $d=16$ (six graphs per cell, equal-variance noise,
$\sigma=1$).

\paragraph{Hardware and compute.} All training runs are batched across the 16 reps within each grid
cell; the correlation/covariance sweep and the least-squares/likelihood extension of
Section~\ref{sec:emp-obstruction2b} were each executed as a single batched job, and the
direction-gap experiment (Section~\ref{sec:emp-obstruction2}) uses
\texttt{scipy.optimize.minimize} with the BFGS method \citep{2020SciPy-NMeth}, run to a gradient
tolerance of $10^{-9}$ independently of the training schedule.

\end{document}